\documentclass{article}

\usepackage{amsmath,amssymb,amsfonts,pifont}
\usepackage{multicol}
\usepackage{amstext}
\usepackage{amsthm}
\usepackage{multirow}
\usepackage{booktabs}
\usepackage[skip=0pt]{subcaption}
\usepackage{times}
\usepackage{lipsum}
\usepackage[shortlabels]{enumitem}
\usepackage{cancel}
\usepackage{wrapfig}
\usepackage{array}
\usepackage{siunitx}
\usepackage{csvsimple}
\usepackage[multidot]{grffile}
\usepackage{bbm}
\usepackage{dblfloatfix}
\usepackage[unicode,psdextra, backref=page]{hyperref}

\usepackage{makecell}
\usepackage{bbm, dsfont}
\usepackage{mathtools}
\usepackage{xcolor}
\usepackage{comment}
\usepackage{blkarray}
\hypersetup{colorlinks = true,linkcolor = blue,anchorcolor =red,citecolor = blue,filecolor = red,urlcolor = red, pdfauthor=author}

\usepackage{geometry}
\usepackage{latexsym}
\usepackage{palatino}
\usepackage{mathpazo}
\usepackage{enumitem}

\usepackage[multiple]{footmisc}
\usepackage{mathrsfs}
\usepackage{tikz}
\usepackage{cleveref}

\usepackage{algpseudocode,algorithm,algorithmicx}
\usepackage{xfrac}

\usepackage{graphicx} %
\usepackage{color}

\usepackage{array}
\usepackage{amssymb}
\usepackage{amsmath}
\usepackage{xspace}
\usepackage{fancyhdr}
\usepackage{comment}

\usepackage[numbers, sort, comma, square]{natbib}
\usepackage{fullpage}
	
\usepackage{xargs}                     \usepackage[colorinlistoftodos,prependcaption,textsize=tiny]{todonotes}

\newcommand{\domain}[1]{\mathsf{dom}\left( #1 \right)}
\renewcommand{\leq}{\leqslant}
\renewcommand{\geq}{\geqslant}
\renewcommand{\le}{\leqslant}
\renewcommand{\ge}{\geqslant}

\hypersetup{final}

\renewcommand{\top}{*}
\newcommand{\meansquared}{\mathsf{err}_{\ell_2^2}}

\newcommand{\counting}{M_{\mathsf{count}}}

\newcommand{\streamlength}{n}

\newcommand{\calD}{\ensuremath{\mathcal{D}}}

\newcommand{\calK}{\ensuremath{\mathcal{K}}}

\newcommand{\calM}{\ensuremath{\mathcal{M}}}

\newcommand{\calQ}{\ensuremath{\mathcal{Q}}}

\newcommand{\calX}{\ensuremath{\mathcal{X}}}

\newcommand{\E}{\mathop{\mathbf{E}}}

\newcommand{\R}{\mathbb{R}}
\newcommand{\real}{\mathbb{R}}
\newcommand{\complex}{\mathbb{C}}
\newcommand{\nat}{\mathbb{N}}

\newtheorem{lem}{Lemma}

\newtheorem{cor}[lem]{Corollary}
\newtheorem{remark}[lem]{Remark}

\newtheorem{fact}[lem]{Fact}
\newtheorem{prop}[lem]{Proposition}

\newtheorem{definition}[lem]{Definition}

\DeclareMathOperator*{\argmin}{arg\,min}

\makeatletter
\newcommand{\vast}{\bBigg@{4}}
\newcommand{\Vast}{\bBigg@{5}}
\makeatother

\newcommand{\ex}[2]{{\ifx&#1& \mathbb{E} \else
\underset{#1}{\mathbb{E}} \fi \left[#2\right]}}
\newcommand{\pr}[2]{{\ifx&#1& \mathbb{P} \else
\underset{#1}{\mathbb{P}} \fi \left[#2\right]}}

\newcommand{\Var}[1]{\ensuremath{\mathbf{Var}\left(#1\right)}}

 \newcommand{\ip}[2]{\left\langle #1, #2\right \rangle}

\newcommand{\onlineAlgorithm}{\mathsf{A}_{\mathsf{oco}}}

\DeclarePairedDelimiter\abs{\lvert}{\rvert}

\DeclarePairedDelimiterX{\infdivx}[2]{(}{)}{%
  #1\;\delimsize\|\;#2%
}

\newcommand{\thetaopt}{\theta_{\mathsf{opt}}}

\newcommand{\regret}[2]{\mathsf{Regret}(#1;#2)\,}

\renewcommand{\epsilon}{\varepsilon}

\renewcommand{\tilde}{\widetilde}

\newcommand{\brak}[1]{{\left\langle {#1} \right\rangle}}
\newcommand{\set}[1]{\left\{ {#1} \right\}}
\newcommand{\norm}[1]{{\left\Vert {#1} \right\Vert}}
\newcommand{\paren}[1]{\left( {#1} \right)}
\newcommand{\sparen}[1]{\left[ {#1} \right]}

\setlist{nolistsep}
\setlist[itemize]{noitemsep, topsep=0pt}

\setlist{nolistsep}
\setlist[itemize]{noitemsep, topsep=0pt}

\newcommand{\op}[1]{\operatorname{#1}}
\newcommand{\trace}{\operatorname{Tr}}

\newcommand{\unitary}[1]{\operatorname{U}\left(#1\right)}
\newcommand{\herm}[1]{\operatorname{Herm}\left(#1\right)}
\newcommand{\pos}[1]{\operatorname{Pos}\left(#1\right)}
\newcommand{\pd}[1]{\operatorname{Pd}\left(#1\right)}
\newcommand{\normal}[1]{\operatorname{N}\left(#1\right)}

\newtheorem{theorem}[lem]{Theorem}

\def\I{\mathbb{1}}

\newcommand{\fa}{\paren{ 2 +  \ln\left( \frac{2n+1}{5} \right) + \frac{\ln(2n+1) }{2 n}}}

\newenvironment{mylist}[1]{\begin{list}{}{
	\setlength{\leftmargin}{#1}
	\setlength{\rightmargin}{0mm}
	\setlength{\labelsep}{2mm}
	\setlength{\labelwidth}{8mm}
	\setlength{\itemsep}{0mm}}}
	{\end{list}}

\usepackage{wrapfig}

\makeatletter
\def\blfootnote{\xdef\@thefnmark{}\@footnotetext}
\makeatother

\usepackage{xargs}                     \usepackage[colorinlistoftodos,prependcaption,textsize=tiny]{todonotes}


\title{Almost Tight Error Bounds on Differentially Private Continual Counting}

\author
{
Monika Henzinger
\thanks{Faculty of Computer Science, University of Vienna.  
email: \texttt{monika.henzinger@univie.ac.at}}
\and 
Jalaj Upadhyay\thanks{Rutgers University. A part of the work was done when the author was visiting Indian Statistical Institute, Delhi. 
email: \texttt{jalaj.upadhyay@rutgers.edu}}
\and
Sarvagya Upadhyay\thanks{Fujitsu Research of America
email: \texttt{supadhyay@fujitsu.com} }
}

\date{}
 
\begin{document}
\maketitle

\begin{abstract}
The first large-scale deployment of private federated learning uses differentially private counting in the {\em continual release model} as a subroutine (Google AI blog titled ``Federated Learning with Formal Differential Privacy Guarantees" on February 28, 2022). For this and several other applications, it is crucial to use a continual counting mechanism with  \emph{small mean squared error}. In this case, a concrete (or non-asymptotic)  bound on the error is very relevant to reduce the privacy parameter $\epsilon$ as much as possible, and hence, it is important to improve upon the constant factor in the  error term. The standard mechanism for continual counting, and the one used in the above deployment, is the {\em binary mechanism}. We present a novel mechanism and show that its mean squared error is both asymptotically optimal and a factor 10 smaller than the error of the binary mechanism. We also show that the constants in our analysis are almost tight by giving non-asymptotic lower and upper bounds that differ only in the constants of lower-order terms. Our mechanism also has the advantage of taking only constant time per release, while the binary mechanism takes $O(\log n)$ time, where $n$ is the total number of released data values. Our algorithm is a matrix mechanism for the \emph{counting matrix}. We also use our explicit factorization of the counting matrix to give an upper bound on the excess risk of the matrix mechanism-based private learning algorithm of  Denisov, McMahan, Rush, Smith, and Thakurta (NeurIPS 2022). 

Our lower bound  for any  continual counting mechanism is the first tight lower bound on continual counting under $(\epsilon,\delta)$-differential privacy and it holds against a non-adaptive adversary. 
{It is achieved using a new lower bound on a certain factorization norm, denoted by $\gamma_F(\cdot)$, in terms of the singular values of the matrix. In particular, we show that for any complex matrix, $A \in \complex^{m \times n}$, 
\[
\gamma_F(A) \geq \frac{1}{\sqrt{m}}\norm{A}_1,
\]
where $\norm{\cdot}$ denotes the Schatten-1 norm.} We believe this technique will be useful in proving lower bounds for a larger class of linear queries. To illustrate the power of this technique, we show the first lower bound on the mean squared error for answering parity queries. This bound applies to the non-continual setting and is asymptotically tight.
\end{abstract}

\thispagestyle{empty}
\clearpage

\tableofcontents
\thispagestyle{empty}
\clearpage
\pagenumbering{arabic}

\section{Introduction}
In recent years, a central problem in federated learning has been to design efficient, differentially private  learning algorithms that can be deployed on a large scale. To solve this problem, many techniques have been proposed that use some variants of {\em differentially private stochastic gradient descent} (DP-SGD) in conjunction with privacy amplification by {\em shuffling}~\cite{prochlo} or {\em sampling}~\cite{bassily2014private}.  However, there are inherent challenges in putting these theoretical ideas to large-scale deployments involving millions of devices\footnote{We refer the interested readers to the Google AI blog for more details regarding the obstacle in the actual deployment of theoretically optimal algorithms like differentially private stochastic gradient descent and one based on amplification by shuffling~\cite{mcmahan2022federated}.}. 
To assuage these issues, a recent line of work leveraged private (online) learning using {\em differentially private follow-the-regularized leader} (DP-FTRL). This particular approach is now employed as a subroutine in the first provably private large-scale deployment by Google for its privacy-preserving federated next-word prediction model~\cite{mcmahan2022federated} (see also the accompanying paper by Kairouz, McMahan, Song, Thakkar, Thakurta, and Xu~\cite{kairouz2021practical} and follow-up work by Denisov, McMahan, Rush, Smith, and Thakurta~\cite{mcmahan2022private}).

A central subroutine used in DP-FTRL is  \emph{differentially private counting under continual observation}, aka \emph{continual counting}~\cite{chan2011private, Dwork-continual}. It refers to the following problem: 
assume an (adaptively generated) binary stream $x = (x_1, x_2, \dots, x_n)$ such that $x_t$ is given in round $t$ (with $1 \le t \le n$), the objective is to continually output in every round $t$, the sum of bits arriving until that round in a differentially private manner. Among many significant advantages of using continual counting for online learning is (a) its resistance to an adaptive choice of the training 
data set and (b) that the privacy proof using this approach does not rely on the convexity of the loss function\footnote{This in particular means that it can be seamlessly extended to neural network where the loss functions are inherently non-convex.}.
These two reasons play a pivotal role in its application in first  production level provable differentially private neural network trained directly on user data~\cite{mcmahan2022federated}.

Continual counting has been used in many other applications as well, including but not limited to, histogram estimation~\cite{cardoso2021differentially, chan2012differentially,huang2021frequency,upadhyay2019sublinear}, non-interactive local learning~\cite{smith2017interaction}, graph analysis~\cite{fichtenberger2021differentially, upadhyay2021differentially}, stochastic convex optimization~\cite{han2022private}, and matrix analysis~\cite{dwork2014analyze, upadhyay2021framework}.
Depending on the downstream use case, the performance of a differentially private continual mechanism is either measured in terms of {\em absolute error} (aka $\ell_\infty$-error) or {\em mean squared error} (aka $\ell_2^2$-error) over the different time steps (defined below).  For continual counting, Fichtenberger, Henzinger, and  Upadhyay~\cite{henzinger2022constant} gave an efficient algorithm based on a subclass of {\em matrix mechanism} known as {\em factorization mechanism} and showed that its $\ell_\infty$-error is almost tight for any matrix mechanism, not only in the asymptotic setting but even \emph{with almost matching constants for the upper and lower bounds.} Concurrently to~\cite{henzinger2022constant}, Denisov, McMahan, Rush, Smith, and Thakurta~\cite{mcmahan2022private} studied the $\ell_2^2$ error for continual counting  and gave
conditions that a factorization has to fulfill to give an optimal $\ell_2^2$-error. They also proposed the use of a 
{\em fixed point algorithms} to compute the factorization, but they do not give an explicit factorization or any provable  $\ell_2^2$-error bound of their mechanism. 
 
On the other hand, given its application in real-world deployments mentioned above, designing an algorithm for continual counting  with provable mean-squared error and one with smallest  constant is highly desirable. The importance of having small constants was also recently pointed out by Fichtenberger, Henzinger, and Upadhyay~\cite{henzinger2022constant} in the continual observation model. {This question was also the center of a subsequent work by Asi, Feldman, and Talwar~\cite{asi2022optimal} on mean estimation in the {\em local model of privacy}.}
An algorithm with small constant in additive error means that we need to use less privacy budget (parameterized by $\epsilon$ and $\delta)$ to guarantee the same accuracy guarantee as  an algorithm with larger constants in the additive error. This has huge impact in practice. For instance, real-world applications  use  prohibitively large values of $\epsilon$ (as large as $19.21$ for the 2021 US Census~\cite{Census2021} and $8.90$ for private learning~\cite{mcmahan2022federated}) to keep the additive error small. In contrast, one  would like $\epsilon$ to be small (ideally $\epsilon \leq 1$) -- using large $\epsilon$ means we need to increase the sample size of the training data, and collecting data is often expensive. {\em Designing a fast mechanism with a small constant factor in the mean squared error is the central topic of this paper.}

Note that there are provable guarantees on the error for the binary mechanism~\cite{chan2011private, Dwork-continual}, but there are two  fundamental issues with the binary mechanism which precludes its application in practice:
\begin{mylist}\parindent
    \item [1.] As we show in Theorem \ref{thm:binarymechanismsuboptimal} and the subsequent paragraph, the mean squared error of the binary mechanism is provably suboptimal.
    \label{item:additive}
    
    \item [2.] The additive  error (even for Honaker's streaming version~\cite{honaker2015efficient}) is non-uniform and depends on the number of $1$'s in the bitwise representation of the current time epoch,
    leading to a non-smooth error function (\cite[Figure 1]{mcmahan2022private}). Consequently, the binary mechanism cannot be used in health-related applications such as ECG monitoring in the Apple watch, where ``smooth'' additive error functions are necessary.
    \label{item:nonuniform}
\end{mylist}

In this paper, we also identify the fundamental reasons why the binary mechanism suffers from the above two major limitations and give algorithm that resolves them while ensuring the advantages of continual counting mentioned earlier so that it can be used in private online learning. More specifically
\begin{mylist}\parindent
    \item [1.] We give matrix mechanism for continual counting that achieves a  mean squared error that is approximately a factor of  $\frac{(\pi \log_2 e)^2}{2} \approx 10.2$ smaller than the binary mechanism. This algorithm can be implemented with quadratic pre-processing time and constant time per round. \label{item:upper}
    \item [2.] We also show that our mechanism is almost optimal by giving the first tight lower bound on the $\ell_2^2$-error of continual counting for \emph{any} mechanism that guarantees $(\epsilon, \delta)$-differential privacy.  Combined with item 1, this resolves the first issue mentioned  above. 
    \item [3.] Our mechanism adds Gaussian noise in a way that makes the error grow smoothly in the number of rounds, which resolves the second fundamental issue mentioned above. 
\end{mylist}

\subsection{Problem Statement and Our  Contributions}
Binary counting is a special type of a \emph{linear query}, which is any linear function $f:\R^\streamlength \rightarrow \R$ of the  $\streamlength$-dimensional input vector $x \in \R^\streamlength$. A fixed set of  $q$ linear queries can be represented in the form of matrix $M\in \R^{q \times \streamlength}$ such that, for any $\streamlength$-dimensional input vector $x \in \R^\streamlength$ (given in a continual or non-continual manner), the answer for query $i$ is $(Mx)[i]$ (the $i$-th coordinate of the vector $Mx$). Then the {\em (additive) mean-squared error} of an $(\epsilon,\delta)$-DP algorithm $\mathcal{M}$ for answering $q$ linear queries on an input $x \in \real^\streamlength$ described by the corresponding matrix, $A$, is 
\begin{align}
\meansquared(\mathcal M,A, \streamlength) = \max_{ x \in \real^\streamlength} \E_{\mathcal M} \sparen{ \frac{1}{\streamlength } \norm{\mathcal{M}(x) - A x}_2^2}.
\end{align}

In this paper, we would be mainly interested in continual counting of a stream of length $n$. Let $A[i,j]$ denote the $(i,j)$-th entry of the matrix $A$, then the mean-squared error for binary counting is  
\begin{align}
\meansquared(\mathcal M,\counting, \streamlength) = \max_{ x \in \{0,1\}^\streamlength} \E_{\mathcal M} \sparen{ \frac{1}{\streamlength } \norm{\mathcal{M}(x) - \counting x}_2^2}, \; \text{where} \; \counting[i,j] = 
\begin{cases}
1 & i \geq j \\
0 & \text{otherwise}
\end{cases}.
\label{eq:meansquared}
\end{align}

Our algorithm is an instantiation of the matrix mechanism~\cite{li2015matrix}, whose mean-squared error can be  bounded in terms of a certain factorization norm, denoted by $\gamma_F(\cdot)$~\cite{edmonds2020power}. Our first set of contributions is concerned with understanding some key properties of this factorization norm for complex matrices. We believe these properties are of independent interest. Then we explore their application in the context of differential privacy.

\subsubsection{Main Result}
\label{sec:properties}
We first define $\gamma_{\op{F}}(\cdot)$ and explain its relationship to the mean squared error, which is the primary reason why we  study $\gamma_{\op{F}}(\cdot)$ and its properties. In the following $\norm{A}_{1\rightarrow 2}$  denotes the maximum of the $2$-norm of the columns of $A$ and $\norm{A}_{\op F}$ is the frobenius norm defined as
$$\norm{A}_{\op{F}} = \left(\sum_{i=1}^{\min\{n,m\}}\sigma_i(A)^2\right)^{1/2} =  \paren{\sum_{i\in[n]}\sum_{j\in[m]} \left\vert A[i,j]\right\vert^2}^{1/2},$$ where $\sigma_i(A)$ is the $i$-th singular value of $A$.

Our mechanism for continual counting is a {\em matrix mechanism}~\cite{li2015matrix}, i.e., a mechanism where, given $M$, we first construct an alternate set of matrices known as {\em strategy matrix} $R$  and {\em reconstruction matrix} $L$ such that $M=LR$.
The strategy matrix is used to generate a private vector, $v$, by adding a Gaussian noise vector to $Rx$. The answer to the original queries are then evaluated from $v$ by computing $Lv$, which can be seen as a  post-processing step. On input $x \in \R^\streamlength$, matrix mechanism outputs the following:
\begin{align*}
\mathcal M_{L,R}(x) = L(Rx + z), \quad \text{where} \quad z \sim N\paren{0, \norm{R}_{1 \to 2}^2 C_{\epsilon,\delta}^2 \I_m}.
\end{align*}

\noindent The privacy proof follows from known results~\cite{mcmahan2022private, li2015matrix}. \noindent For a matrix $M \in \complex^{n \times m}$, let us define\footnote{Edmonds, Nikolov, and Ullman~\cite{edmonds2020power} defined $\gamma_{\op F}(M) = \min\set{\frac{1}{\sqrt{\streamlength}}\norm{L}_{\op F} \norm{R}_{1 \to 2} : M = LR}$ for a matrix $A \in \complex^{n \times d}$. We prefer the definition in \cref{eq:gammanorm} as it is more aligned with the definition of such norms in functional analysis and operator algebra.}
\begin{align}
    \gamma_{\op F}(M) = \min \set{ \norm{L}_{\op F} \norm{R}_{1 \to 2}: LR=M}.
\label{eq:gammanorm}
\end{align}

Now, if $\mathcal M_{L,R}$ is a matrix mechanism that uses the factorization $M =LR$, then using Li, Miklau, Hay, McGregor, and Rastogi~\cite{li2015matrix}, we have
\begin{align}
\meansquared(\mathcal M_{L,R},M, \streamlength) = {\frac{1}\streamlength } C_{\epsilon,\delta}^2 \norm{L}_{\op{F}}^2 \norm{R}_{1 \to 2}^2.
\label{eq:meansquarederror}    
\end{align}

In particular, for an optimal choice of $L$ and $R$
\begin{align}
\label{eq:meansquaredgammanorm}
\meansquared(\mathcal M_{L,R}, M,\streamlength) = {\frac{1}\streamlength }C_{\epsilon,\delta}^2 \gamma_{\op{F}}(M)^2.
\end{align}

We also investigate $\gamma_{\op{F}}(\cdot)$ in more detail for general complex matrices and show many useful properties in \Cref{app:gammanorm}. 
These are properties that may be of independent interest considering that $\gamma_{\op{F}}(\cdot)$ can be used to characterize the mean-squared error of linear queries~\cite{edmonds2020power}.  
One of the main properties of $\gamma_{\op{F}}(\cdot)$ is that it can be characterized as a semidefinite programming (SDP) problem. We give the details in \Cref{sec:sdpgammanorm}. 

The SDP characterization allows us to prove many key properties of $\gamma_{\op F}(\cdot)$, which can be of independent interest. In particular, it allows us to prove the following key lemma that relates the $\gamma_F(\cdot)$ to the spectrum of the matrix {(also proved in Li and Miklau~\cite{li2013optimal} using a different proof technique\footnote{ Aleksandar Nikolov informed us about Li and Miklau~\cite{li2013optimal}'s paper after the first publication of this work.})}.
\begin{lem}
\label{lem:gammanormlowerboundmain}
For a matrix $A\in\complex^{n\times m}$, we have
    $ \frac{\norm{A}_1}{\sqrt{m}} \leq \gamma_{\op F}(A) \leq \norm{A}_{\op{F}}$,  where $\norm{A}_1$ is the Schatten-$1$ norm (or, trace norm) of $A$. In particular, if $A$ is unitary, $\gamma_{\op F}(A) = \sqrt{n}$ and if all singular values of $A$ are same, then $\gamma_{\op F}(A) = \norm{A}_F$. 
\end{lem}

Since the lower bound for mean-squared error for a set of linear queries can be stated in terms of the $\gamma_{\op F}(\cdot)$ of the corresponding query matrix, \Cref{lem:gammanormlowerboundmain} provides an easier method to prove lower bounds. We explore two applications of this lower bound  in this paper: continual counting and parity queries. A proof of this lemma is presented in \Cref{sec:proofmainlemma}.

The semidefinite characterization also allows us to show many useful facts about $\gamma_{\op{F}}(\cdot)$, which we believe can be of independent interest. 
\begin{mylist}{\parindent}
    \item [1.] The optimal factorization can be achieved by finite dimensional matrices. This is a direct consequence of strong duality of the SDP of $\gamma_{\op F}(\cdot)$. We show strong duality in \Cref{lem:slater}.
    
    \item [2.] For a matrix $A\in\complex^{n\times m}$, there exist $B\in\complex^{n\times p}$ and $C\in\complex^{p\times m}$ for $p\le m$ such that $A=BC$ and $\gamma_{\op{F}}(A) = \norm{B}_{\op{F}}\norm{C}_{1\to 2}$. If $A$ is a real matrix, then we can assume without loss of generality that $B$ and $C$ are real matrices too.
    
\end{mylist}

\noindent While we establish properties of $\gamma_{\op{F}}(.)$ for complex matrices, the second item allows us to assume that real matrices will have optimal real factorization. This is important for privacy applications where theoretical results have been established assuming real factorization. 
These and other useful properties of $\gamma_{\op F}(.)$ are proved in \Cref{app:gammanorm}. 

\subsubsection{Other Contributions}

\paragraph{Contribution 1: An almost exact error bound for continual counting.}

The classic algorithm for differentially private counting under continual observation is the \emph{binary (tree) mechanism}~\cite{chan2011private,Dwork-continual}. 
With Laplacian  noise they show \emph{for each  round} that the additive $\ell_\infty$-error is $O(\log^{3/2} \streamlength)$ with constant probability, which requires the use of a union bound over all $\streamlength$ updates and results in an  $\ell_{\infty}$-error of $O(\log^{5/2} \streamlength)$.
However, with Gaussian noise  an $\ell_{\infty}$-error of $O(\log^{3/2}\streamlength)$
can be achieved~\cite{jain2021price}.
Neither work gives a bound on the $\ell_2^2$-error although an $O(\log^2(\streamlength))$ bound is implicit in some works~\cite{kairouz2021practical}. Note that the concurrent and independent work by Denisov, McMahan, Rush, Smith, and Thakurta~\cite{mcmahan2022private} do not give any bounds on the additive error of their matrix  mechanism based algorithm and only show empirical improvement.

Our algorithm factorizes the matrix $\counting$ in terms of two lower triangular matrices $L$ and $R$,
i.e.~$\counting=LR$, and we show that $\norm{L}_{\op F} \norm{R}_{1 \to 2} \le 
\sqrt n\left(1 + \frac{\ln(4n/5)}{\pi}\right)$. This  immediately implies an  upper bound on $\gamma_{\op F}(\counting)$.

In particular, we show the following in \Cref{sec:upperboundcounting}:  
\begin{theorem}
\label{thm:counting}
For any $0 <\epsilon,\delta <1$, there is an efficient $(\epsilon,\delta)$-differentially private continual counting algorithm $\mathcal M_{\op{fact}}$, that on receiving a binary stream of length $\streamlength$, achieves the following error bound:  
 \begin{align}
 \begin{split}
 \meansquared(\mathcal M_{\op{fact}}, \counting, \streamlength) \leq C_{\epsilon,\delta}^2  \left(1  + \frac{\ln(4\streamlength/5)}{\pi} \right)^2, \quad  \text{where}~~C_{\epsilon,\delta} = \frac{2}{\epsilon} \sqrt{{\frac{4}{9} +  \ln \paren{\frac{1}{\delta}\sqrt{\frac{2}{\pi}}}}}
 \end{split}
 \label{eq:mainupperbound}
 \end{align}  
is the variance required by the Gaussian mechanism to preserve $(\epsilon,\delta)$-differential privacy and $\ln(\cdot)$ denotes the natural logarithm. The mechanism requires $O(n^2 + ns)$ preprocessing time and constant time per update round, where $s$ is the time required to sample from zero-mean unit variance Gaussian distribution.
\end{theorem}

We also show an almost tight lower bound on $\gamma_{\op F}(\counting)$. This implies that for any matrix mechanism based 
algorithm~\cite{li2015matrix} for continual counting and for small enough $(\epsilon,\delta)$, our bound is almost tight. While the limitation to matrix mechanism based algorithms seems  restrictive, \emph{all currently known} mechanisms for continual observation fall under this class of mechanism (see~\Cref{sec:suboptimality} for an explanation). A full proof of Theorem~\ref{thm:lowerboundgammanorm} is presented in \Cref{sec:lowerboundmain}.

\begin{theorem}
[Lower bound on matrix mechanisms.] 
\label{thm:lowerboundgammanorm}
For any  $\epsilon > 0$ and $0 \le \delta \le 1$,
let $\mathfrak M$ be the set of $(\epsilon,\delta)$-differentially private continual counting algorithms  that use the matrix mechanism. Then  
\begin{align*}
      \min_{\mathcal M \in \mathfrak M} \meansquared(\mathcal M, \counting, \streamlength) \geq \frac{C_{\epsilon,\delta}^2 }
      {\pi^2}\fa^2
      .
\end{align*}
\end{theorem}

Note that the constants in Theorem~\ref{thm:counting} and \ref{thm:lowerboundgammanorm} match exactly for the $\left(\ln(n)\right)^2$ term and the bounds only differ in the constants in lower-order terms. More concretely, for all  $n \leq 2^{50}$, the additive gap between the upper bound (Theorem~\ref{thm:counting}) and lower bound (Theorem~\ref{thm:lowerboundgammanorm}) is at most  $10C_{\epsilon,\delta}^2$. 

\paragraph{Contribution 2: A lower bound on the $\ell_2^2$-error for any mechanism for continual counting.}
Theorem~\ref{thm:lowerboundgammanorm} precludes an improvement using matrix mechanism, but does not preclude algorithms using a more careful choice of noise addition as the only known lower bound for countinual counting is $\Omega(\log(n))$ for {$\ell_\infty$-error}  when $\delta=0$~\cite{Dwork-continual}. More generally, there is no lower bound known on the $\ell_2^2$-error and $\delta \neq 0$.  These facts lead to the natural question, recently also asked  by Denisov, McMahan, Rush, Smith, and Thakurta~\cite{mcmahan2022private}:
 {\em Is there a mechanism that is not factorization-based and achieves a better mean-squared error?}
We show this is not the case by proving the following theorem in \Cref{sec:lowerboundproofcounting}, which also implies that our mechanism is asymptotically optimal.
\begin{theorem}
[Lower bound on the $\ell_2^2$-error of continual counting]
\label{thm:lowerboundadditive}
For any  $\epsilon > 0$ and $0 \leq \delta < \frac{c}{2e^\epsilon}$ for some absolute constant $c>0$,
let $\mathfrak M$ be the set of $(\epsilon,\delta)$-differentially private algorithms for counting under continual observation. 
Then for all $n$,
\begin{align}
      \min_{\mathcal M \in \mathfrak M} \meansquared(\mathcal M, \counting, \streamlength) \geq \frac{1}{(e^{4\epsilon} - 1)^2 \pi^2 }\fa^2      .
     \label{eq:lowerbound}    
\end{align}
 Further, if $\mathfrak M$ is a set of $(\epsilon,\delta)$-differentially private mechanism for continual counting that add noise oblivious of the input for  $(\epsilon,\delta)$ small enough constant, then we can improve the dependency on privacy parameter:  
 \[
 \min_{\mathcal M \in \mathfrak M} \meansquared(\mathcal M, \counting, \streamlength) \geq \frac{1}{(e^{2\epsilon} - 1)^2 \pi^2 }\fa^2.
 \]
\end{theorem}

\paragraph{Contribution 3: Suboptimality of the binary mechanism.}
Few natural questions to ask are whether we can improve the accuracy of the binary mechanism using a better analysis and how much worse the additive factor in the binary mechanism is than our mechanism. We answer these in the following theorem: 

\begin{theorem}
\label{thm:binarymechanismsuboptimal}
Let $\mathcal M_B$ be the binary (tree) mechanism~\cite{chan2011private, Dwork-continual} that adds noise sampled from an appropriate Gaussian distribution to every node of the binary tree. Let $\calM_{\op{fact}}$ be our mechanism guaranteeing Theorem~\ref{thm:counting}. Then
\begin{align*}
\frac{\meansquared(\mathcal M_B,\counting,n)}{\meansquared(\mathcal M_{\op{fact}},\counting,n)} \geq  \frac{\log_2(n)\paren{ 1 + \log_2(n)}}{2\paren{1 + \frac{\ln(4n/5)}{\pi}}^2} . 
\end{align*}
Let $\mathfrak M$ be the set of $(\epsilon,\delta)$-differentially private continual counting algorithms  that use a matrix mechanism, and let $\calM \in \mathfrak M$ be a matrix mechanism that achieves the optimal error stated in Theorem~\ref{thm:lowerboundgammanorm}. Then
\begin{align*}
\frac{\meansquared(\mathcal M_B,\counting,n)}{\meansquared(\mathcal M,\counting,n)} =  \frac{\pi^2\log_2(n)\paren{ 1 + \log_2(n)}}{2 \fa^2}. 
\end{align*}
\end{theorem}

\begin{wrapfigure}{r}{0.5\textwidth}\label{fig:diff}
\centering    \includegraphics[width=0.48\textwidth]{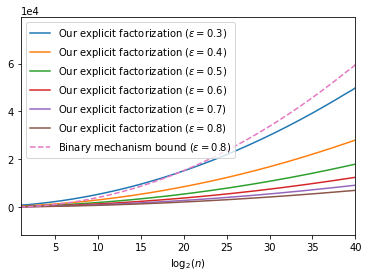}
\vspace{-5mm} 
\caption{Comparison of $\ell_2^2$-error bounds of binary and our factorization based mechanisms for varying $\epsilon$ and $\delta = 10^{-10}$. Our mechanism ($\epsilon=0.3$) incurs less error than binary mechanism ($\epsilon=0.8$) for $n\ge 2^{19}$.}
    \label{fig:binaryversusourbound}
   \vspace{-5mm}
\end{wrapfigure}

In short, the above theorem states that the mean-squared error of binary mechanism is suboptimal by approximately a factor $\frac{1}{2}(\frac{\pi}{\ln 2})^2 \approx 10.2$.  A plot of this comparison is given in \Cref{fig:binaryversusourbound} for varying $\epsilon$ and $\streamlength$ up to $2^{40}$ with $\delta= 10^{-10}$. A proof of Theorem~\ref{thm:binarymechanismsuboptimal} is presented in \Cref{sec:suboptimalbinarymechanism}. 

\medskip
\paragraph{Contribution 4: Online Private Learning.}
\label{sec:practicalimpact}{}
A major application of differentially private continual counting with bounded mean squared error is private learning~\cite{kairouz2021practical}. Here, the goal is to minimize the {\em excess risk}, i.e., either minimize the {\em average loss} on the given data ({\em excess empirical risk}) or minimize the error on ``unseen" data ({\em excess population risk}). Smith and Thakurta~\cite{guha2013nearly} introduced a technique for online private learning using continual counting as a subroutine. In \emph{online} learning, we aim to bound the {\em regret}, i.e., the average loss incurred over all rounds compared to the post-hoc optimal decision (see \Cref{def:regret}). One can then use the standard technique of {\em online-to-batch} conversion to get a bound on population risk from the regret bound. This approach was recently used by Kairouz, McMahan, Song, Thakkar, Thakurta, and Xu~\cite{kairouz2021practical} and a follow-up work by Denisov, McMahan, Rush, Smith, and Thakurta~\cite{mcmahan2022private} -- the difference being that ~\cite{kairouz2021practical} used the binary mechanism as a subroutine and got a provable regret bound while~\cite{mcmahan2022private} suggested the matrix mechanism and show regret improvements only empirically. These algorithms have direct practical applications, see the extensive experiments mentioned in~\cite{mcmahan2022private} and the Google AI blog detailing the use of binary mechanism in their recent deployment~\cite{mcmahan2022federated}. Thus,  it is important to have  provable guarantees on such an algorithm including constant factors. Since the {\em online-to-batch} conversion is standard, we focus only on giving the regret bound. We show the following non-asymptotic bound on the algorithm of Denisov, McMahan, Rush, Smith, and Thakurta~\cite[Algorithm 1]{mcmahan2022private}
 with the continual counting algorithm implemented by our mechanism:
\begin{theorem}
\label{thm:privopt}
Let $\mathcal K$ be a closed, convex, and compact set and $\mathcal D$ be the data universe. Further, let $\ell:\mathcal K \times \mathcal D \to \real$ be $1$-Lipschitz with respect to $\ell_2$ norm and convex  in the first parameter, 
i.e., $\norm{\ell(x;\cdot) - \ell(x';\cdot)}_2 \leq \norm{x-x'}_2$ for all $x,x' \in \calK$.
Then there is an efficient $(\epsilon,\delta)$-differentially private online algorithm, $\onlineAlgorithm$, that on a stream of data $x_1,\cdots, x_n$,  outputs private models $[\theta_1, \cdots, \theta_\streamlength]$ such that, for any $\thetaopt \in \mathcal K$:
\begin{align}
\regret{\onlineAlgorithm}{n} \leq  \norm{\thetaopt}_2 \sqrt{\frac{\paren{1 + \frac{\ln(4n/5)}{\pi}}(1 + C_{\epsilon,\delta}  \sqrt{d})} {2\streamlength}},
\label{eq:regretmaintheorem}    
\end{align}
where
\[
\regret{\onlineAlgorithm}{\streamlength} = \frac{1}{\streamlength} \sum_{i=1}^\streamlength \ell(\theta_t;x_t) - \frac{1}{\streamlength} \min_{\theta \in \mathcal{K}} \sum_{i=1}^\streamlength \ell(\theta;x_i).
\]
Furthermore, the regret bound in \cref{eq:regretmaintheorem} holds even if the data points $x_t$ are picked adversarially.
\end{theorem}

This result shows that our algorithm improves the algorithm in Kairouz, McMahan, Song, Thakkar, Thakurta, and Xu~\cite[Theorem 4.1]{kairouz2021practical} by a constant factor of about $3$ (computed from the constants in their proof)
 and  helps explain the empirical observation made in Denisov, McMahan, Rush, Smith, and Thakurta~\cite{mcmahan2022private}, who reported that the matrix mechanism based stochastic gradient descent ``{\em significantly improve the privacy/utility curve (in fact, closing 2/3rds of the gap to non-private training left by the previous state-of-the-art for single pass algorithms)}", where the previous state-of-the art algorithm refers to the one by Kairouz, McMahan, Song, Thakkar, Thakurta, and Xu~\cite{kairouz2021practical}. A proof of Theorem~\ref{thm:privopt} is given in \Cref{sec:privopt}.

\begin{remark}
As noted in~\cite{mcmahan2022private}, private learning algorithms that use continual counting
are also flexible to the various settings studied in Kairouz, McMahan, Song, Thakkar, Thakurta, and Xu~\cite{kairouz2021practical}; therefore, our results extend seamlessly to adversarial regret for \emph{composite loss functions, excess risk}, and various practical extensions such as \emph{heavy ball momentum}. We refer the interested readers to the relevant sections in~\cite{mcmahan2022private,kairouz2021practical}. Further, our factorization is also digaonally dominant allowing the efficient computation in  practical settings such as in \cite[Appendix F]{mcmahan2022private}.
\end{remark}

\medskip
\paragraph{Contribution 5: Lower Bounds on Special Linear Queries} Our lower bound technique is actually quite general and will most likely have further applications. To 
exhibit the generality of our lower bound technique, we use it to show another lower bound, this time in the non-continual setting. Specifically, we give the first lower bound on the mean-squared error for \emph{parity queries} in the batch, i.e., non-continual setting, where the underlying data does not change.

\begin{definition}[Parity Query]
Let $d$ and $w$ be integer parameters and let the domain be $\calX = \set{\pm 1}^d$. Then a parity query is a query that belongs to the family of queries
\begin{align}
\calQ_{d,w} = \set{ q_P(x) = \prod_{i \in P} x_i : P \subset \set{1, \cdots, d}, |P| = w }.
\label{eq:parity}    
\end{align}    
\end{definition}

Note that parity queries are  important in data analysis.
We show the following bound for parity queries, which to our knowledge, is the first lower bound on the mean-squared error for parity queries issued by a non-adaptive adversary under $(\epsilon,\delta)$-differential privacy and it is tight: Answering parity queries with the Gaussian mechanism achieves the same dependency in terms of $d$ and $w$ as our lower bound.
\begin{theorem}
\label{thm:parity}
Let $\epsilon > 0$, let $0 \le \delta \le 1$ and let $\calQ_{d,w}$ be the class of parity queries defined in \cref{eq:parity}. Then for any $(\epsilon,\delta)$-differentially private mechanism $\calM$ that takes as input $d$ and $w$, and that can answer any query of $\calQ_{d,w}$,  it holds that
\[
\meansquared\paren{\calM, \calQ_{d,w}, \binom{d}{w}} \geq C_{\epsilon}^2  {\binom{d}{w}}.
\]
\end{theorem}
A proof of Theorem~\ref{thm:parity} is given in \Cref{sec:lowerboundmain}. 
A corresponding bound for $\ell_\infty$ error was computed by Edmonds, Nikolov, and Ullman~\cite{edmonds2020power}. This completes the picture for parity queries.

\subsection{Our Techniques}
We fix some notation that we use in this section (detail notations are in \Cref{sec:notations}). For a matrix $X$  and a vector $v$, let $X^*$ and $v^*$ denote their complex-conjugates (when $X$ is a real matrix, then $X^*$ is the transposed matrix), respectively. For a complex number, $z = a + \iota b \in \complex$, let $|z|$ denote its {\em modulus}, $a^2 +b^2$. Moreover, let $\I_k$ denote a $k\times k$ identity matrix, $J_{k,l}$ denote the all ones $k\times l$ matrix, and $1_k$ denote the vector of all ones in $\real^k$. For any two matrices $P, Q\in \complex^{n\times m}$, let $P\bullet Q$ denote their Schur (Hadamard) product. We use $A \succeq 0$ to denote that $A$ is a positive semidefinite (PSD) matrix and $A \succeq B$ to denote that $A-B \succeq 0$. We use $w \in \real_{++}^n$ to denote a strictly positive vector.
Finally, for any matrix $Y\in\complex^{n\times m}$, let $\widehat Y$ denote the following Hermitian matrix:
\begin{align*}
\widehat Y =\begin{pmatrix}
0^{n \times n} & Y  \\
Y^\top  & 0^{m \times m}
\end{pmatrix}.
\end{align*} 

\subsubsection{Main Result: Useful Properties of $\gamma_{\op F}(\cdot)$}  
Fix $A\in\complex^{n\times m}$ for which we wish to characterize $\gamma_{\op F}(.)$ as an SDP\footnote{We give the SDP characterization for complex matrices which will involve Hermitian matrices. However, as stated above, when $A$ is a real matrix, one can  without loss of generality consider an SDP involving symmetric matrices.}.
Note that, for any factorization $A=LR$, we can assume that $\norm{L}_{\op F} = \norm{R}_{1\to 2}$ by appropriate rescaling. That is, for an optimal factorization $A=LR$ with respect to $\gamma_{\op F}(A)$, we can assume that 
\begin{align}
\gamma_{\op F}(A) = \norm{L}_{\op F}^2 = \norm{R}_{1\to 2}^2.\label{eq:optfactorization}
\end{align}
Any factorization of $A = LR$ can be turned into a PSD matrix satisfying the following matrix constraint: 
\[
X = 
\begin{pmatrix}
X_1 & X_2  \\
X_2^*  & X_3
\end{pmatrix} \succeq 0
\qquad \text{such that} \qquad 
X_2 = A = LR.
\]
The fact that $X \succeq 0$  means that $X_1 = LL^*$ and $X_3 = R^*R$. 
The first implication of this fact is that $\trace(X_1) = \trace(LL^*) = \norm{L}_{\op F}^2$, where $\trace(Z)$ denotes the sum of diagonal entries of any square matrix $Z$. The second implication is that the $i$-th diagonal entry of $X_3$, denoted by $X_3[i,i]$, is exactly the squared $2$-norm of the $i$-th column of $R$. The maximum of the $2$-norm over each column of $R$ is exactly $\norm{R}_{1\to 2}$ (see  Fact~\ref{fact:2-normcol}). From \cref{eq:optfactorization}, we wish to minimize $\norm{L}_{\op F}^2$ such that every diagonal entry of the matrix $X_3$ is at most $\norm{L}_{\op F}^2$. Hence, we can rewrite $\gamma_{\op F}(A)$ as minimizing a real number $\eta$ such that, for the matrix 
\[
X = 
\begin{pmatrix}
X_1 & X_2  \\
X_2^*  & X_3
\end{pmatrix} \succeq 0
\qquad \text{satisfying} \qquad 
X_2 = A = LR,
\qquad \text{we have} \qquad
\trace(X_1) = \eta \quad \text{and} \quad X_3[i,i] \leq \eta
\]
for every $i\in \{1, \dots, m\}$. This can be described as an SDP as described in \Cref{fig:sdp} (the primal problem).
\begin{figure}[t]
    \centering
    \fbox{
    \begin{minipage}{.4\textwidth}
    \centering
    \underline{Primal}
    \begin{align*}
    \gamma_{\op{F}}(A):=        \min \quad & \eta \\
        \text{s.t.} \quad & \sum_{i=1}^{n} X[i,i] = \eta \\
        & X[i,i] \leq \eta \quad \forall~ n+1 \leq i \leq n+m \\
        & X \bullet \widehat J_{n,m} = \widehat A \\
        & X \succeq 0.
    \end{align*}
    \end{minipage}%
    
    \begin{minipage}{0.4\textwidth}
    \centering
    \underline{Dual}
    \begin{align*}
    \gamma_{\op{F}}(A)=    \max \quad & w^* (\widehat A \bullet \widehat Z) w \\
        \text{s.t.} \quad & \begin{pmatrix} n\I_n & 0 \\ 0 & \I_m \end{pmatrix} \succeq \widehat Z \\
        & w =  \begin{pmatrix}
        w_1 \\ w_2
        \end{pmatrix} \in \real_{++}^{n+m} \\
        &\norm{w}_2 = 1 \\
        & w_1 = \alpha 1_n
    \end{align*}
    \end{minipage}
    }
    \caption{SDP for $\gamma_{\op F}(\cdot)$ norm.}
    \vspace{-2mm}
    \label{fig:sdp}
\end{figure}

As SDPs come in primal-dual pairs,  any feasible solution of the primal problem is an upper bound on $\gamma_{\op F}(A)$. Similarly, any feasible solution of the dual problem is a lower bound on $\gamma_{\op F}(A)$. We will utilize this fact to show the desired lower bound as stated in \Cref{lem:gammanormlowerboundmain}. We give the detail proof in \Cref{sec:proofmainlemma}. 
Note that the dual problem as stated in \Cref{fig:sdp} is in a form that will be helpful in proving the lower bound. A rigorous explanation of how we arrive at this formulation is described in Appendix~\ref{app:dualgammanorm}.


\subsubsection{Other Contributions}

\medskip
\paragraph{Contribution 1: An almost exact error bound for continual counting.}
As described above, we analyze the matrix mechanism given in  \Cref{alg:factorizationmechanism}. 
From \cref{eq:meansquaredgammanorm}, the question of determining lower and upper bounds on the $\ell_2^2$-error for continual counting reduces to a purely linear algebraic problem of estimating $\gamma_{\op F}(\counting)$. There are many ways of estimating this quantity. One particular way to bound it is by using the {\em completely bounded spectral norm}~\cite{paulsen1982completely}:
\[
\norm{A}_{\op{cb}} := \min \set{ \norm{L}_{2 \to \infty} \norm{R}_{1 \to 2}: A = LR},
\]
where $\norm{L}_{2 \to \infty}$ is the maximum of the $2$-norm of the rows of $L$. The $\norm{\cdot}_{\op{cb}}$ norm plays an important role in bounding the $\ell_\infty$-error~\cite{henzinger2022constant}. It has been extensively studied in operator algebra and tight bounds are known for $\norm{\counting}_{\op{cb}}$~\cite{mathias1993hadamard}. However, using known bounds for $\norm{\counting}_{\op{cb}}$ does not yield a tight bounds on $\gamma_{\op F}(\counting)$. It is known that, for a matrix $A \in \complex^{n \times n}$, $\norm{A}_{\op{cb}}  \le \gamma_{\op{F}}(A) \le \sqrt{n}\norm{A}_{\op{cb}}$, and as we will show later, the gap between $\norm{\counting}_{\op{cb}}$ and $\gamma_{\op F}(\counting)$ is indeed approximately $\sqrt{n}$. 
Hence, we utilize different techniques, as described below, to show the following almost tight bounds on $\gamma_{\op F}(\counting)$: 
\begin{align}
\frac{1}{\pi}\fa
\leq \frac{1}{\sqrt{\streamlength}}\gamma_{\op{F}}(\counting)  \leq  \left(1  + \frac{\ln (4\streamlength/5)}{\pi} \right).
\label{eq:gammanormboud}
\end{align}

The error bounds of Theorem~\ref{thm:counting} and Theorem~\ref{thm:lowerboundgammanorm} follows by combining \cref{eq:meansquaredgammanorm} and \cref{eq:gammanormboud}.
Furthermore the factorization achieving the upper bound is given by two lower triangular matrices $L$ and $R$. 

While the  matrix mechanism  presented in \Cref{alg:factorizationmechanism} requires $O(t)$ time at round $t$ ($L_t$ on line 5 can be computed in time $O(t)$ using \cref{eq:entriesL}), we 
show how to modify it to achieve  constant time per round and $O(\streamlength^2)$ pre-processing time.
The main idea is to sample a  vector $g$ from a suitable distribution during pre-processing, multiplying $L$ with $g$, and storing the resulting vector $z = Lg$. When bit $x_t$ is released, 
the mechanism simply adds $z[t]$ to the true answer. We also show how to adjust the privacy proof to this setting using Theorem~\ref{thm:denisovadaptive} in \Cref{sec:upperboundcounting}.
Thus, in what follows, we just discuss the technique to prove \cref{eq:gammanormboud}.

\begin{algorithm}[t]
\caption{Matrix Mechanism for Continual Counting, $\calM_{\op{fact}}$}
\begin{algorithmic}[1]
   \Require A stream of bits $(x_1,\cdots, x_\streamlength)$, length of the stream $\streamlength$, $(\epsilon,\delta)$: privacy budget.
   \State Define a function $f: \mathbb Z \to \real$ as follows:
    \begin{align}
    f(k)= 
    \begin{cases}
    0 & k <0 \\
    1 & k = 0\\
    \left(1- \frac{1}{2k}\right) f(k-1) & k\geq 1\\
    \end{cases}
    \label{eq:entriesL}
    \end{align}
    \State Let $L,R \in \real^{\streamlength \times \streamlength}$ be matrices with entries as follows:
   $L[i,j] = f(i-j)$ {and} $R[i,j] = f(i-j).$ 
   \For {$t$ in $1, 2, \cdots, n$}
        \State Sample $z \sim N\paren{0, C_{\epsilon,\delta}^2 \norm{R}_{1 \to 2}^2 \I_t}$.  
        \State Define a $t$-dimensional row vector $L_t = \begin{pmatrix}
        L[t,1] & L[t,2] & \cdots & L[t,t] 
        \end{pmatrix}$. \label{step:L_t}
        \State Receives $x_t$ and output 
        \[
        a_t = \paren{\sum_{i=1}^t x_t} + \ip{L_t}{z}
        \]
   \EndFor
\end{algorithmic}
\label{alg:factorizationmechanism}
\end{algorithm}

\paragraph{Upper bound in \cref{eq:gammanormboud}.}
We bound $\norm{R}_{1 \to 2}$ and $\norm{L}_{\op F}$ for $L$ and $R$ computed in \Cref{alg:factorizationmechanism}\footnote{Recently, Amir Yehudayoff (through Rasmus Pagh) communicated to us that this factorization was stated in the 1977 work by Bennett~\cite[page 630]{bennett1977schur}}. 
We bound $\norm{R}_{1 \to 2}$ using the fact that entries of $R$ can be represented as a double factorial allowing us to use Theorem~\ref{thm:double_factorial} to get $\norm{R}_{1 \to 2}^2 \leq \left(1 + \frac{\ln(4n/5)}{\pi}\right)$. To bound $\norm{L}_{\op F}$, we use the fact that $L$ is a lower-triangular matrix and the $\ell_2^2$ norm of the $m$-th row of $L$ (denoted by $L[m,:]$) is the same as the $\norm{L(m)}_{2 \to \infty}^2 = \norm{L(m)}_{1 \to 2}^2$ norm of the $m \times m$ principal submatrix, $L(m)$, of $L$. That is, $ \norm{L[m,:]}_2^2 \leq \left(1 + \frac{\ln(4m/5)}{\pi}\right)$. In particular, we can bound $$\norm{L}_F^2 = \sum_{m=1}^\streamlength \norm{L[m,:]}^2_2 \leq \sum_{m=1}^\streamlength \left(1 + \frac{\ln(4m/5)}{\pi}\right) \le n \left(1 + \frac{\ln(4n/5)}{\pi}\right).$$
A complete proof is presented in \Cref{sec:upperboundcounting}.

\medskip
\paragraph{Lower bound in \cref{eq:gammanormboud}.}
We begin with a brief outline of the algebraic method used to lower bound  $\gamma_{\op{F}}(\counting)$. For a general matrix $A$, it is possible to show that $\gamma_{\op F}(A)$ is lower bounded by the square of the sum of a subset $S$ of singular values. If the singular values are ordered in descending order $\sigma_1 \geq \sigma_2 \geq \cdots \geq \sigma_n$, then the subset $S$ contains exactly the singular values $\sigma_1, \sigma_3,\cdots$.  This is a consequence of Cauchy-Schwarz and Weyl's inequalities. With this at our disposal, we can then use standard results on the singular values of $\counting$ to give a lower bound. However, this does not yield a tight bound. In particular, the slackness in the lower bound results from the application of Cauchy-Schwarz and Weyl's inequalities \cite{merikoski2004inequalities} in the first step.  To overcome this slackness, we take an optimization perspective and use our SDP characterization of $\gamma_{\op F}(\cdot)$. 
We note that such an optimization perspective has been taken in some of the earlier works (see~\cite{edmonds2020power} and references therein). 
We present a complete proof of Theorem~\ref{thm:lowerboundgammanorm}  in \Cref{sec:lowerboundmain}.

\medskip
\paragraph{Contribution  2: Lower bound for any mechanism for continual counting.}
Our lower bound on mechanism for continual counting that uses factorization of $\counting$ follows from our lower bound on $\gamma_{\op F}(\counting)$. To extend this to a lower bound on all $(\epsilon,\delta)$-differentially private mechanism for continual counting, we use the lower bound  on the mean-squared error 
by Edmonds, Nikolov, and Ullman~\cite{edmonds2020power} (see the proof of Theorem~\ref{thm:factorization}): the lower bound on
any $(\epsilon,\delta)$-differentially private mechanism for any linear query defined by a matrix $A$ is at least $C_{\epsilon}^2 \gamma_{\op F}(A)^2/n$, where $C_{\epsilon} = \frac{1}{e^{2\epsilon}-1}$. We note that the value of the constant $C_{\epsilon}$ can be found in a lower bound by Kasivishwanathan, Rudelson, Smith, and Ullman~\cite{kasiviswanathan2010price}. 
Combining this with  \cref{eq:gammanormboud}, we get 
Theorem~\ref{thm:lowerboundadditive}.

\medskip
\paragraph{Contribution 3: Suboptimality of the binary mechanism.}
The binary mechanism returns in each round $t$ the sum of $O(\log n)$ sub-sums, called \emph{p-sums}, depending on the number of bits set in the binary representation of $t$. Now each row of the right factor $R_{\mathsf{binary}}$ is used to sum up each $p$-sum, while each row of the left factor $L_{\mathsf{binary}}$ is used to compute the sum of the $O(\log n)$ $p$-sums. More formally, the right factor $R_{\mathsf{binary}}$ is constructed as follows: $R_{\mathsf{binary}} = W_m$ where $W_1, \cdots, W_m$ are defined recursively as follows:
\begin{align*}
W_1 = \begin{pmatrix}
1
\end{pmatrix}, 
\quad W_k = \begin{pmatrix}
W_{k-1} & 0 \\
0 & W_{k-1} \\
1_{2^{k-2}} & 1_{2^{k-2}} 
\end{pmatrix}, \quad k \leq m.
\end{align*}

Note that $R_{\mathsf{binary}} = W_m$ is a matrix of $\set{0,1}^{\streamlength \times (2\streamlength -1)}$ matrix,
with each row corresponding to the $p$-sum computed by the binary mechanism. 
The corresponding matrix $L_{\mathsf{binary}}$ is a matrix of
$\set{0,1}^{\streamlength \times (2\streamlength-1)}$,
where row $t$ 
has $\log_2 (t)$ entries, corresponding exactly to the binary representation of $i$.
Computing the $\norm{L_{\mathsf{binary}}}_F$ and $\norm{R_{\mathsf{binary}}}_{1 \to 2}$ leads to the bounds stated in the theorem, which combined with our results, implies the suboptimality of the binary mechanism.

\medskip
\paragraph{Contribution 5.} Our SDP-based lower bound technique is very general and can be applied even in the non-continual counting, as we show by using it to give a lower bound for parity queries. In particular, e use the observation of Edmonds, Nikolov, and Ullman~\cite{edmonds2020power} that the query matrix corresponding to any set of the parity queries is the  $\binom{d}{w}$ matrix formed by taking the corresponding rows of the $2^d \times 2^d$ unnormalized Hadamard matrix. Let us call this matrix $S$.  The lower bound then follows by computing the Schatten-$1$ norm of $S$. We present a complete proof in \Cref{sec:paritylowerbound}.

\medskip

\paragraph{Outline of the paper.} We give all necessary notation and preliminaries in \Cref{sec:notations} and present the semidefinite program for $\gamma_{\op{F}}$ in \Cref{sec:sdpgammanorm}. \Cref{sec:boundsgammabound} contains the upper and lower bound on $\gamma_{\op{F}}(\counting)$ and also the more efficient mechanism, thus giving the proof of both Theorem~\ref{thm:counting}, Theorem~\ref{thm:lowerboundgammanorm},  and Theorem~\ref{thm:lowerboundadditive}. In \Cref{sec:suboptimality} we show that every known mechanism for continual counting is a matrix mechanism and give lower bounds for the mean squared error of the binary mechanism, thereby proving Theorem~\ref{thm:binarymechanismsuboptimal}. \Cref{sec:applications} contain all upper and lower bounds for the further applications. \Cref{app:dualgammanorm} gives the dual characterization of $\gamma_{\op F}(.)$ and \Cref{app:gammanorm} covers the useful properties and bounds on $\gamma_{\op F}(.)$. 

\section{Notations and Preliminaries}\label{sec:notations}
We use  $\nat$ to denote the set of natural numbers,  $\mathbb Z$ to denote the set of integers, $\real$ to denote the set of real numbers, $\real_+$ to denote set of non-negative real numbers, $\real_{++}$ to denote set of positive real numbers, and $\complex$ to denote the set of complex numbers. For $n, m\in \nat$ such that $m\le n$, we let $[n]$ denote the set $\{1, \dots, n\}$, and $[m,n]$ denote the set $\{m, \dots, n\}$. We will reserve the lower-case alphabets $n,m,p,q,r$ for describing natural numbers and $i,j,k$ for indexing. We fix the symbol $\streamlength$ to denote the length of the stream.

\subsection{Linear Algebra}\label{sec:linal}
In this section, we review linear algebra and relevant facts and describe the notations used throughout the paper.

\medskip
\paragraph{Vector spaces and norms.} 
We denote $n$-dimensional real vector space and complex vector space by $\real^n$ and $\complex^n$, respectively. The non-negative orthant and the set of $n$-dimensional strictly positive vectors in $\real^n$ are denoted $\real^n_+$ and $\real^n_{++}$, respectively. We will reserve the lower-case alphabets $u,v,w,x,y,z$ to denote vectors in $\real^n$ or $\complex^n$. The $i$-th co-ordinate of a vector $v$ is denoted $v[i]$ and the set $\{e_i: i\in[n]\}$ denote the set of standard basis vectors. We reserve $1_n$ to denote the vector of all $1$'s in $\real^n$.
When a complex (real) vector space is equipped by a inner product, it is called a complex (real) inner product space. The canonical inner product associated with a complex vector space $\complex^n$ is defined as 
\[
\ip{u}{v} = \sum_{i\in[n]} u^*[i] v[i],
\]
for any two vectors $u,v\in\complex^n$ where $u^*$ is the vector whose entries are complex-conjugate of the entries of vector $u$. 
The following norms will be used in this paper (termed as $\ell_2, \ell_1,$ and $\ell_\infty$ norms, respectively):
\[
\norm{u}_2 = \sqrt{\ip{u}{u}} \qquad \text{and} \qquad
\norm{u}_1 = \sum_{i\in[n]}\left\vert u[i]\right\vert 
\qquad \text{and} \qquad
\norm{u}_{\infty} = \max_{i\in[n]}\left\{\left\vert u[i]\right\vert\right\}.
\]
In one of our applications, we will also need the concept of {\em dual norm}. 
\begin{definition}
[Dual norm]
Let $\norm{\cdot}$ be any norm on $\mathcal{K}$. Then its associated {\em dual norm} is defined as follows:
\[
\norm{z}_\star := \sup\set{\ip{z}{x} : \norm{x} \leq 1 }
\]
\end{definition}

\noindent It is easy to see that $\ell_2$ norm is dual of itself and $\norm{\cdot}_1$ is the dual of $\norm{\cdot}_\infty$. 

\medskip
\paragraph{Matrices.}
The vector space of complex $n \times m$ matrices is denoted by $\complex^{n \times m}$. The set of real $n \times m$ matrices form a subspace of $\complex^{n \times m}$ and is denoted $\real^{n \times m}$. For a matrix $A$, its $(i,j)$-th entry is denoted by $A[i,j]$, the $i$-th row is denoted $A[i;]$, and the $j$-th column is denoted $A[;j]$. 
We use the notation $J_{n,m}$ to denote an all one $n \times m$ matrix,  $J_n$ to denote $J_{n,n}$,
$\I_n$ to denote the $n \times n$ identity matrix,
and $0^{n \times m}$ to denote an $n \times m$ all zero matrix. 

The complex-conjugate of $A$ is denoted by $A^*$. The complex-conjugate of a real matrix $B$ is the transpose of the matrix itself, and we will employ the notation $B^*$ to denote the transposed matrix. We will be mostly referring to the following classes of matrices in the remainder of this paper.

\begin{mylist}{\parindent}
    \item[1.] A matrix $A\in\complex^{n\times n}$ is {\it normal} if $AA^* = A^*A$. The set of normal matrices is denoted $\normal{\complex^n}$. The eigenvalues of $A$ can be complex. The singular values of $A$ are just the absolute value of the respective eigenvalues.
    
    \item[2.] A normal matrix $U$ is {\it unitary} if it also satisfies $UU^* = \I_n$, where $\I_n$ is the $n \times n$ identity matrix. The set of unitary matrices is denoted $\unitary{\complex^n}$. The eigenvalues of a unitary matrix lie on the unit circle in a complex plane. In other words, every singular value of a unitary matrix is $1$.
    
    \item[3.] A matrix $A\in\complex^{n\times n}$ is {\it Hermitian} if $A = A^*$. The set of Hermitian matrices is denoted $\herm{\complex^n}$. If the entries of a Hermitian matrix $A$ are real, we call the matrix {\it symmetric}. The eigenvalues of a Hermitian matrix are real. 
    
    \item[4.] A Hermitian matrix $A$ is {\it positive semidefinite} if all its eigenvalues are non-negative. The set of such matrices is denoted $\pos{\complex^n}$. The notation $A\succeq 0$ indicates that $A$ is positive semidefinite and the notations $A\succeq B$ and $B\preceq A$ indicate that $A - B\succeq 0$ for Hermitian matrices $A$ and $B$.
    
    \item[5.] A positive semidefinite matrix $A$ is {\it positive definite} if all its eigenvalues are strictly positive. The set of such matrices is denoted $\pd{\complex^n}$. The notation $A\succ 0$ indicates that $A$ is positive definite and the notations $A\succ B$ and $B\prec A$ indicate that $A - B\succ 0$ for Hermitian matrices $A$ and $B$.
    
\end{mylist}

\begin{remark}
\label{rem:singularevals}
For any matrix $A\in\complex^{n\times n}$, its singular values and eigenvalues are denoted by the sets $\left\{\sigma_i(A): i\in[n]\right\}$ and $\left\{\lambda_i(A): i\in[n]\right\}$. Throughout this paper we follow the following convention.
\begin{mylist}{\parindent}
\item [1.] For a matrix $A$, its singular values are  sorted in descending order. That is, $\sigma_1(A) \ge \dots \ge \sigma_n(A) \ge 0$. The eigenvalues of $A$ are ordered with respect to the ordering of the singular values of $A$. In other words, $\lambda_i(A)$ corresponds to the $i$-th singular value in the sorted list.\label{item1}

\item[2.] We will adopt a different convention for Hermitian matrices. For a Hermitian matrix $A\in\herm{\complex^n}$, the eigenvalues are real and are sorted in descending order: $\lambda_1(A) \ge \dots \ge \lambda_n(A)$. 
\end{mylist}
\end{remark}

\noindent For any matrix $A\in\complex^{n \times m}$, we reserve the notation $\widehat A$ to denote the following matrix:  
\begin{align}\label{eq:hatnotation}
\widehat A = \begin{pmatrix}
0^{n \times n} & A \\
A^* & 0^{m \times m}
\end{pmatrix}.
\end{align}
The matrix $\widehat A$ is a Hermitian matrix (or symmetric, if $A\in\real^{n\times n}$). The trace of a square matrix $A\in\complex^{n\times n}$ is denoted $\trace(A)$ and equals $\sum\limits_{i \in [n]}A[i,i]$. For two matrices $A,B \in \complex^{n \times n}$, their inner product is defined as
\[
\ip{A}{B} = \trace\left(A^*B\right).
\]
For a vector $v\in\complex^n$, we let $\Delta_n:\complex^n\rightarrow \complex^{n\times n}$ denote the map that maps an $n$-dimensional vector to a diagonal matrix with diagonal entries being the entries of the vector. For two matrices $A, B \in \complex^{n \times m}$, we denote their Hadamard (or Schur) product as $A\bullet B$. We list the following well known properties of the Schur product that is used in this paper.

\begin{prop}
\label{prop:dotproductJAB}
Let $A,B \in \complex^{n \times n}$. Then $\ip{J_n}{A\bullet B} = \ip{A}{B} = \trace(A^*B)$. Moreover, if $A\bullet B \in \herm{\complex^n}$, then $\ip{A\bullet B}{J_n} = \ip{A}{B} = \trace(A^*B)$.
\end{prop}
\begin{proof}
Let $D = A \bullet B$. Then  
$
\ip{J_n}{D} = \trace(J_n D) = \sum_{i,j\in[n]}D[i,j] = \sum_{i,j\in[n]}A[i,j] B[i,j] = \ip{A}{B} = \trace(A^*B).
$
If $A\bullet B \in \herm{\complex^n}$, it holds that $\ip{J_n}{A\bullet B} = \ip{A\bullet B}{J_n}$, and the result follows from above.
\end{proof}

\begin{prop}
\label{prop:schurhermitianunitvector}
Let $A,B \in \herm{\complex^{n}}$ and $v \in \complex^n$. Then $\ip{A}{B\bullet vv^*} = v^*(A\bullet B)v$. 
\end{prop}
\begin{proof}
Unraveling the formula
\[
\ip{A}{B\bullet vv^*} = \sum_{i,j\in[n]}A[i,j]B[i,j]v[i]v^*[j] = 
\sum_{i,j\in[n]}v^*[j]A^*[j,i]B^*[j,i]v[i] = v^* (A^*\bullet B^*) v = v^* (A\bullet B)v.
\]
This completes the proof of Proposition~\ref{prop:schurhermitianunitvector}.
\end{proof}

\begin{prop}\label{prop:schurdiagonalequiv}
Let $x,y \in \real^n$, and $D_x$ and $D_y$ be diagonal matrices formed by vectors $x$ and $y$, respectively, Then for any matrix $A\in \complex^{n \times n}$, $D_x A D_y = A\bullet xy^*$. 
\end{prop}
\begin{proof}
A straightforward calculation will show that the $(i,j)$ entry of the matrix $D_xAD_y$ is $A[i,j]x[i]y[j]$. This implies that $D_xAD_y = A \bullet xy^*$ and the proposition follows.
\end{proof}
\noindent We need the following definition and the subsequent well known lemma for our proof.
\begin{definition}
[Schur's complements]
Let $A\in\complex^{n\times n}$, $B\in \complex^{n\times m}$, $C\in \complex^{m\times n}$, and $D\in\complex^{m\times m}$ be matrices and let
\begin{align}
\label{eq:Schurblockmatrix}
S = \begin{pmatrix}
A & B \\
C & D
\end{pmatrix}.
\end{align}
Then the Schur complements of $S$ are the matrices 
\begin{align}
\label{eq:Schurcomplement}
S_A =  D - CA^{-1}B \quad (\text{if} \; A^{-1} \; \text{exists})
\qquad \qquad \text{and} \qquad \qquad
S_D =  A - BD^{-1}C \quad (\text{if} \; D^{-1} \; \text{exists}).
\end{align}
\end{definition}
\noindent A simple calculation shows that if $A^{-1}$ exists then
\[
\begin{pmatrix} A & B \\ C & D \end{pmatrix} 
=
\begin{pmatrix} \I_{n} & 0 \\ CA^{-1} & \I_m \end{pmatrix} 
\begin{pmatrix} A & 0 \\ 0 & S_A  \end{pmatrix} 
\begin{pmatrix} \I_n & A^{-1}B \\ 0 & \I_m \end{pmatrix}, 
\]
and if $D^{-1}$ exists then
\[
\begin{pmatrix} A & B \\ C & D \end{pmatrix} 
=
\begin{pmatrix} \I_{n} & BD^{-1} \\ 0 & \I_m \end{pmatrix} 
\begin{pmatrix} S_D & 0 \\ 0 & D  \end{pmatrix} 
\begin{pmatrix} \I_n & 0  \\ D^{-1}C & \I_m \end{pmatrix}.
\]
In particular, if $S$ is Hermitian, we have the following equivalent characterization for $S\in\pos{\complex^{n+m}}$ and $S\in\pd{\complex^{n+m}}$.
 
\begin{lem}
\label{lem:schurcomplement}
Let $S$ be defined as in \cref{eq:Schurblockmatrix} assume it is a Hermitian matrix. 
Then we have the following.
\begin{mylist}\parindent
    \item[1.]
    Suppose $A\in\pd{\complex^n}$. Then $S\succeq 0$ if and only if $S_A \succeq 0$. Moreover, $S\succ 0$ if and only if $S_A \succ 0$.
    \item[2.]
    Suppose $D\in\pd{\complex^m}$. Then $S\succeq 0$ if and only if $S_D \succeq 0$.  Moreover, $S\succ 0$ if and only if $S_D \succ 0$.
\end{mylist}
\end{lem}

\paragraph{Matrix norms.}
We begin with defining matrix norms induced by vector norms. For a matrix $A\in\complex^{n \times m}$, the norm $\norm{A}_{p\rightarrow q}$ is defined as
\[
\norm{A}_{p\rightarrow q} = \max_{x\in \complex^m}\left\{\frac{\norm{Ax}_q}{\norm{x}_p}\right\}.
\]
Of particular interests are the norms $\norm{A}_{1\rightarrow 2}$ and $\norm{A}_{2\to\infty} $, which are the maximum of the $2$-norm of the columns of $A$ and the maximum of the $2$-norm of the rows of $A$, respectively.

In this paper, we work with $\norm{A}_{1\rightarrow 2}$. For the sake of completion, we show that the assertion we made above is true (the proof that  $\norm{A}_{2\to\infty} $ is the maximum of the $2$-norm of the rows of $A$ follows similarly). 
\begin{fact}\label{fact:2-normcol}
For a matrix $A\in \complex^{n\times m}$, the norm $\norm{A}_{1\rightarrow 2}$ is the maximum 2-norm of the columns of A.
\end{fact}
\begin{proof}
To see why this is true, let us fix $x\in \complex^m$ such that $\norm{x}_1 = 1$. We have that
\begin{align*}
\norm{Ax}_2^2 
& = \sum_{i\in[n]}\sum_{j\in[m]} \left\vert A[i,j]\right\vert^2 \left\vert x[j]\right\vert^2 
\leq \sum_{i\in[n]}\sum_{j\in[m]} \left\vert A[i,j]\right\vert^2 \left\vert x[j]\right\vert \\
& = \sum_{j\in[m]}\left(\sum_{i\in[n]}\left\vert A[i,j]\right\vert^2\right)\left\vert x[j]\right\vert 
\leq \max_{j\in[m]} \left(\sum_{i\in[n]}\left\vert A[i,j]\right\vert^2\right)
\end{align*}
where both the inequalities follows because $\norm{x}_1 = 1$. If $k$ is the column of $A$ with maximum $2$-norm, it is clear that the maximum is achieved by setting $x = e_k \in \real^m$.
This complete the proof of Fact~\ref{fact:2-normcol}.
\end{proof}

\noindent We will employ the following Schatten norms in this paper:
\begin{align}
\label{eq:schattennorm}
\norm{A}_{\infty} = \sigma_1(A)
\quad \text{and} \quad
\norm{A}_1 = \sum_{i=1}^{p}\sigma_i(A)
\quad \text{and} \quad
\norm{A}_{\op{F}} = \left(\sum_{i=1}^{p}\sigma_i(A)^2\right)^{1/2} = \paren{\sum_{i\in[n]}\sum_{j\in[m]} \left\vert A[i,j]\right\vert^2}^{1/2}
\end{align}
where $p=\min\{m,n\}$. Finally, we state the factorization norm that is used to derive our bounds. Given a matrix $A\in\complex^{n \times m}$, we define $\gamma_{\op{F}} (A)$ as 
\[
\gamma_{\op{F}}(A) = \inf\left\{ \norm{B}_{\op{F}}\norm{C}_{1\rightarrow 2} 
\; : \; A = BC \right\}.
\]
The quantity $\gamma_{\op{F}}(.)$ is a norm and can be achieved by a factorization that involves finite-dimensional matrices. Moreover, if $A$ is a real matrix, then we can restrict our attention to the real factorization of $A$. We refer interested readers to Appendix~\ref{app:gammanorm} for more detail. Another factorization norm that we mention in this work is cb-norm (also known as $\gamma_2 (A)$ norm) which is defined as
\begin{align}
\norm{A}_{\mathsf{cb}} = \gamma_2(A) := \inf\left\{ \norm{B}_{2\to\infty}\norm{C}_{1\rightarrow 2} 
\; : \; A = BC \right\}.
\label{eq:cbnorm}
\end{align}
By construction the two aforementioned factorization norms satisfy the following relationship
\begin{align}
    \label{eq:gammanormrelationship}
    \gamma_2(A) \le \gamma_{\op{F}}(A) \le \sqrt{n}\gamma_2(A)
\end{align}
where $A\in\complex^{n\times m}$. Moreover, both inequalities are tight -- the left inequality is an equality when $A$ has only one entry and the right inequality is an equality for all unitary matrices.

\medskip
\paragraph{Matrix decompositions.}
At various points in this paper, we will refer to one of the following types of matrix decompositions.

\begin{mylist}{\parindent}
    \item[1.] \emph{Singular value decomposition}: Any complex matrix $A \in \complex^{n \times m}$ can be decomposed as $A = U\Sigma_A V^*$, where $\Sigma_A\in\pd{\complex^p}$ is a diagonal matrix of strictly positive singular values of $A$ (and hence, $p \le \min\{n,m\}$), and $U\in\complex^{n\times p}$ and $V\in\complex^{m\times p}$ satisfying $U^*U = V^*V = \I_p$. Moreover, $UU^* \preceq \I_n$ and $VV^* \preceq \I_m$. For a real matrix $A \in \complex^{n \times m}$, one can assume that $U$ and $V$ are matrices with real entries.
    
    \item[2.] \emph{Spectral decomposition}: Any normal matrix $A\in\normal{\complex^n}$ can be decomposed as $A = U\Lambda_A U^*$, where $\Lambda_A$ is a diagonal matrix of eigenvalues of $A$ and $U$ is a unitary matrix. Moreover, any positive semidefinite matrix $A\in\pos{\complex^n}$ can be decomposed as $A = BB^*$ for $B \in\complex^{n\times n}$. If $A$ is a real positive semidefinite matrix, then one can assume that $B$ is a real matrix.
\end{mylist}

\noindent We state the following well known linear algebra facts about Hermitian and normal matrices.

\begin{fact}
\label{fact:tracePSD}
Let $A \in \herm{\complex^n}$ with eigenvalues $\left(\lambda_1(A), \cdots, \lambda_n(A)\right)$ and $B\in\pos{\complex^n}$. 
Then 
\[
\trace(A) = \sum_{i=1}^\streamlength \lambda_i(A)
\qquad \text{and} \qquad
\trace(AB) \geq \lambda_n(A)\trace(B).
\]
\end{fact}

We need the following result regarding the singular values of $\counting$:
\begin{theorem}
\label{thm:matouvsek2020factorization}
Let $\counting$ be the matrix defined in \cref{eq:meansquared}. Let $\sigma_1, \cdots, \sigma_\streamlength$ be its $n$-singular values. Then for all $1\leq i \leq \streamlength$,
\[
\sigma_i = \frac{1}{2} \left\vert\csc \paren{\frac{(2i-1)\pi}{4\streamlength+2}} \right\vert.
\]
\end{theorem}
\begin{proof}
The proof argument is due to Gilbert Strang. We present a proof for the sake of completion. We can compute $\counting^{-1}$ exactly as follows:
\begin{align*}
    \counting^{-1}[i,j] = \begin{cases}
    -1 & i = j+1 \\
    1 & i =j \\
    0 & \text{otherwise}
    \end{cases}.
\end{align*}
From this, we can compute $(\counting^\top \counting)^{-1}$ as follows:
\[
(\counting^\top \counting)^{-1} = 
\begin{pmatrix}
2 & -1 & 0 & \cdots & 0 \\
-1 & 2 & -1 & \cdots & 0 \\
\vdots & \vdots & \vdots & \ldots & \vdots \\
0 &  \cdots & -1 & 2 & -1 \\
0 & 0  & \cdots & -1 &  1
\end{pmatrix}.
\]
This is exactly the graph Laplacian matrix $B$ with Neumann boundary conditions considered in \cite[Section 9]{strang2014functions}. The result follows.
\end{proof}

\begin{remark}
Throughout the paper, we implicitly assume that $\counting \in\real^{\streamlength\times \streamlength}$ and make use of $L$ and $R$ to denote the factorization of $\counting$, i.e., $\counting = LR$. 
\end{remark}

\noindent For the upper bound, we use the following result on the double factorial.
\begin{theorem}
[Chen and Qi~\cite{chen2005best}]
\label{thm:double_factorial}
For any $m \in \mathbb N$, let $(m)!!$ denote the double factorial defined as follows:
\begin{align*}
(2m)!! = \prod_{i=1}^m (2i) \qquad \text{and} \qquad
(2m-1)!! = \prod_{i=1}^m (2i-1).
\end{align*}
Then 
\begin{align*}
\sqrt{\frac{1}{\pi(k + \frac{4}{\pi} -1 )}} \leq \frac{(2k-1)!!}{(2k)!!} < \sqrt{\frac{1}{\pi(k + \frac{1}{4})}}.
\end{align*}
Furthermore, the constants $\frac{4}{\pi}-1$ and $\frac{1}{4}$ are tight. 
\end{theorem} 

\paragraph{Vector calculus.}
We give a brief overview of the vector calculus required to understand this paper. The basic object  of  concern is a {\em vector field} in a  space. In this paper, we will always be concerned with the $d$-dimensional vector space defined over reals, $\real^d$. A vector field is an assignment of a vector to each point in a space. Unlike scalar calculus, in vector calculus, we can have  various differential operators, which are typically expressed in terms of the {\em del} operator, $\nabla$. 

Given a scalar field, $f$, i.e., a scalar function of position $\theta \in \real^d$,  its gradient at any point $\theta \in \real^d$, denoted by $\nabla_\theta f(\theta)$, is defined as the vector field
\[
\begin{pmatrix}
\frac{\partial}{\partial \theta_1} f(\theta) \\
\frac{\partial}{\partial \theta_2} f(\theta) \\ 
\vdots \\
\frac{\partial}{\partial \theta_d} f(\theta) \\
\end{pmatrix}.
\]

A {\em Hessian matrix} is a square matrix of second-order partial derivatives of a scalar-valued function. At a point $\theta \in \real^d$, the Hessian of a scalar field, $f$,  is 
\[
 \begin{pmatrix}
\frac{\partial^2 f(\theta)}{\partial \theta_1^2} & \frac{\partial^2 f(\theta)}{\partial \theta_1 \partial \theta_2} & \cdots & \frac{\partial^2 f(\theta)}{\partial \theta_1 \partial \theta_d} \\
\frac{\partial^2 f(\theta)}{\partial \theta_1 \theta_2} & \frac{\partial^2 f(\theta)}{ \partial \theta_2^2} & \cdots & \frac{\partial^2 f(\theta)}{\partial \theta_2 \partial \theta_d} \\
\vdots & \vdots & \ddots & \vdots  \\
\frac{\partial^2 f(\theta)}{\partial \theta_n \partial \theta_1} & \frac{\partial^2 f(\theta)}{\partial \theta_n \partial \theta_2} & \cdots & \frac{\partial^2 f(\theta)}{\partial \theta_n^2}
\end{pmatrix}
\]
We use the symbol $\nabla^2 f(\theta)$ to denote the Hessian of the scalar field $f$. If the second partial derivatives are all continuous, then the Hessian matrix is a symmetric matrix. This fact is known as  {\em Schwarz's theorem.} 

We defined these concepts for $\theta \in \real^d$. They generalize naturally when $\theta \in \mathcal K$ for some closed compact set $\mathcal K \subseteq \real^d$~\cite{boyd2004convexopt}.

\subsection{Convex Optimization}
In this section, we give a brief overview of convex optimization to the level required to understand \Cref{sec:privopt}. Let $\mathcal K$ denote a convex, closed and compact set over which the optimization problem is defined.  

\begin{definition}
[Extended-value convex function]
An {\em extended-value convex} function $\phi: \mathcal{K} \to \real \cup \set{\infty}$ satisfies
\[
\phi(\alpha x + (1-\alpha) y) \leq \alpha \phi(x) + (1-\alpha) \phi(y)
\]
for all $\alpha \in (0,1)$ and the domain of the $\phi$ is 
$\domain{\phi} := \set{x: \phi(x) < \infty}$.
\end{definition}

\begin{definition}
[Proper function and Convex function]
A function $\phi$ is {\em proper} if there exists an $ x \in \mathcal{K}$ such that $\phi(x) < +\infty$ and, for all $ x \in \mathcal{K}, \phi(x) >-\infty$. A {\em convex function} is an extended-value convex function that is also proper.
\end{definition}

\begin{definition}
[Subgradient] The subgradient (or subdifferential) set of a convex function $\phi: \mathcal{K} \to \real \cup \set{\infty}$  at a point $x$ is defined as follows:
\[
\partial \phi(x):= \set{g: \forall y \in \mathcal{K}, \quad \phi(y) - \phi(x) \geq \brak{g,(y - x)}}.
\]
If the function is continuously differentiable, then one of the subgradients is the {\em gradient} of the function and denoted by $\nabla \phi$. 
\end{definition}
Note that the subdifferential is a non-empty set if $x$ is in the strict interior of the domain of $\phi$ and is defined even if the function is not continuously differentiable. 

\begin{definition}
[Strongly convex]
Let $\alpha >0$. A {convex function} $\phi: \mathcal{K} \to \real \cup \set{\infty}$ is  an {\em $\alpha$-strongly convex function with respect to the $\ell_2$-norm}  if for all $x,y \in \mathcal{K}$,
\[
\forall g \in \partial \phi(x), \quad \phi(y) - \phi(x) \geq \brak{g, y-x} + \frac{\alpha}{2} \norm{y-x}^2_2.
\]
\end{definition}

\noindent The following is shown in McMahan~\cite{mcmahan2017survey}:
\begin{lem}
[Lemma 7 in McMahan~\cite{mcmahan2017survey}]
\label{lem:mcmahan}
Let $\phi_1:\mathcal{K} \to \real \cup \set{\infty}$ be a convex function such that 
\[
x = \argmin_{x \in \mathcal K} \phi_1(x).
\]
Let $\psi$ be a convex function such that $\phi_2 (x) = \phi_1(x) + \psi(x)$ is  $\lambda$-strongly convex with respect to the norm $\norm{\cdot}_2$. Let 
\[
y = \argmin_{x \in \mathcal K} \phi_2(x).
\]
Then for any $b \in \partial \psi(x)$, we have 
\[
\norm{x - y}_2 \leq \frac{1}{\lambda}\norm{b}_2, \quad \text{and} \quad \forall \overline x \in \mathcal K,~ \phi_2(x) - \phi_2(\overline x) \leq \frac{1}{2\lambda}\norm{b}_2^2.
\]
\end{lem}

\begin{fact}
\label{fact:strongconvexity}
If the function $\phi: \mathcal{K} \to \real \cup \set{\infty}$ is twice differentiable, i.e., admits a second derivative, then $\alpha$-strong convexity is equivalent to $\alpha \I_d \preceq \nabla^2 \phi(x)$, where $\nabla^2 \phi(x)$ denotes the Hessian\footnote{A Hessian is a square matrix of second-order partial derivatives of a scalar valued function.} of the function $\phi$ at $x \in \mathcal{K}$.
\end{fact}

In this paper, we will extensively use duality theory. Central to it is the Fenchel conjugate, which generalizes Lagrangian duality. 
\begin{definition}
[Fenchel conjugate]
The {\em Fenchel conjugate} of an arbitrary function $\psi:\mathcal{K} \to \real \cup \set{\infty}$ is defined as follows:
\[
\psi^\star(g) := \sup_{x} \brak{g,x} - \psi(x).
\]
\end{definition}

One of the main motivations of our work is to get an exact bound on regret minimization for convex optimization in the online setting using the private online convex optimization algorithm of Kairouz, McMahan, Song, Thakkar, Thakurta, and Xu~\cite{kairouz2021practical}. In online optimization, an online player makes decisions iteratively. After committing to the decision, the player suffers a loss. This loss is made known to the player only after the decision is made. The goal of the player is to ensure that the total loss, i.e, the average
of the losses of all decisions, known as {\em regret}, is minimized compared to the loss of the  post-hoc best decision, which is the decision which generates the smallest total loss if used for \emph{all} online choices. We will consider the following setting of online convex optimization. 

Let $\mathcal D$ denote the domain of data samples and let  $\ell: \mathcal{K} \times \mathcal D \to \real$ be a convex function in the first parameter. Then the main goal of an online algorithm is to minimize the {\em regret} against an arbitrary post-hoc optimizer $\theta^{\mathsf{opt}} \in \mathcal{K}$:
\begin{definition}
[Regret minimization]
\label{def:regret}
Let $\onlineAlgorithm$ be an online convex programming algorithm, which at every step $t \in [\streamlength]$, observes data samples $[x_1, \cdots x_{t-1}]$ and outputs $\theta_t \in \mathcal{K}$. The performance of $\onlineAlgorithm$ is measured in terms of {\em regret} over $\streamlength$ iterations:
\[
\regret{\onlineAlgorithm}{\streamlength} = \frac{1}{\streamlength} \sum_{i=1}^\streamlength \ell(\theta_t;x_t) - \frac{1}{\streamlength} \min_{\theta \in \mathcal{K}} \sum_{i=1}^\streamlength \ell(\theta;x_i).
\]
\end{definition}

\noindent Note that we are using the definition of regret which is normalized instead of the one used in Hazan~\cite{hazan2019introduction}. One can consider both {\em adversarial regret}~\cite{hazan2019introduction}, where the data
sample $x_t$ are drawn adversarially based on the past outputs $\set{\theta_1, \cdots, \theta_{t-1}}$, and {\em stochastic regret}~\cite{hazan2019introduction}, where the data is sampled i.i.d. from some fixed unknown distribution $\calD$. 

\paragraph{Follow-the-regularized leader.} One of the most important and successful families of low-regret algorithms for online convex optimization
is the {\em follow-the-regularized leader} (FTRL).
The generic FTRL meta-algorithm is defined in \Cref{alg:ftrl}.  Different FTRL algorithms use a different \emph{regularization function}  (aka \emph{regularizer}, leading to different update rules. 
Two most common regularizer are (1) the {\em entropy function}, which results in the multiplicative weight update method and used in the private multiplicative weight update method, and (2) the $\ell_2^2$-regularizer, which is the choice of regularizer used in this paper and whose privacy guarantee is well studied~\cite{guha2013nearly}. 
\begin{algorithm}[t]
\caption{Follow-the-regularized leader~\cite[Algorithm 10]{hazan2019introduction}}
\begin{algorithmic}[1]
   \Require A convex closed compact set $\mathcal K$,  regularization function $\rho: \mathcal K \to \real$, dimension $d$ of the parameter space, learning rate $\eta>0$.
   \State Set $\theta_1 = \argmin_{\theta \in \mathcal K} \rho(\theta)$.
   \For {$t=1$ to $\streamlength$}
        \State Predict $\theta_t$.
        \State Observe the new data sample $x_t$ and compute $\nabla_t = \nabla \ell(\theta_t;x_t)$.
        \State Update
        \[
        \theta_{t+1} = \argmin_\theta \set{\eta \sum_{i=1}^t \ip{\nabla_i}{\theta} + \rho(\theta)}
        \]
   \State return $\sum_{i=1}^{t+1} \theta_i$. 
    \EndFor
\end{algorithmic}
\label{alg:ftrl}
\end{algorithm}

We use the following bound on the regularized follow-the-regularized leader (\Cref{alg:ftrl}):
\begin{theorem}
[Theorem 5.2 in \cite{hazan2019introduction}]
\label{thm:hazan2019introduction}
Let $R: \mathcal K \to \real$ be the regularization function and let $\eta>0$ be the learning rate. The regularized follow-the-perturbed leader (defined in \Cref{alg:ftrl}) attains for every $u \in \mathcal K$, the following bound on the regret:
\begin{align}
\regret{\onlineAlgorithm}{\streamlength} \leq \frac{2\eta}\streamlength \sum_{i=1}^\streamlength \paren{\norm{\nabla_i}^2_2} + \frac{\rho(u)-\rho(\theta_1)}{\streamlength \eta}.
\label{eq:regrethazan2019introduction}
\end{align}
\end{theorem}

The first term (summation of the so called {\em local norms}, $\norm{\nabla_i}^2_2$, of the gradients) in \cref{eq:regrethazan2019introduction} is called the {\em width term}  and the second term is known as the {\em diameter term}. If we have a universal bound on the local norms,  i.e., for all $1\leq i \leq n$, $\norm{\nabla_i}_2^2 \leq L$ for some constant $L$, then we can optimize over the learning rate $\eta$ to get the final regret.

\subsection{Differential Privacy}
The privacy definition we use in this paper is {\em differential privacy}. We define it next based on the notion of \emph{neighborhood} which we define below for different applications.
\begin{definition}
[Differential privacy]
Let $\mathcal M : X \rightarrow R$  be a randomized algorithm mapping from a domain $X$ to a range $R$. $\mathcal{M}$ is $(\epsilon,\delta)$-differentially private if for every all neighboring dataset $D$ and $D'$ and every measurable set $C \subseteq R$,
$$\mathsf{Pr} [\mathcal{M} (D) \in C] \leq e^\epsilon \mathsf{Pr} [\mathcal{M} (D') \in  C] + \delta.$$
\end{definition}
Central to the notion of privacy is the notion of neighboring dataset. In this paper, we use the standard notion of neighboring dataset for each use case.
\begin{enumerate}
    \item \emph{Continual observation}: Two streams, $S = (x_1,\cdots, x_\streamlength) \in \set{0,1}^\streamlength$ and $S' = (x_1',\cdots, x_\streamlength') \in \set{0,1}^\streamlength$ are neighboring if there is at most one $1 \leq i \leq n$ such that $x_i \neq x_i'$. This is known as {\em event level privacy}~\cite{chan2011private, Dwork-continual}.
    
    \item \emph{Online convex optimization}: Two dataset $D = \set{x_1,\cdots, x_\streamlength}$ and $D' = \set{x_1',\cdots, x_\streamlength'}$ are considered neighboring if they differ in one data-point~\cite{guha2013nearly}. That is, there is at most one $1\leq i \leq \streamlength$ such that $x_i \neq x_i'$. 
    
    \item {\em Parity Queries}: Two dataset $D = \set{x_1,\cdots, x_\streamlength} \in \set{-1,+1}^\streamlength$ and $D' = \set{x_1',\cdots, x_\streamlength'}\in \set{-1,+1}^\streamlength$ are neighboring if there is at most one $1\leq i \leq \streamlength$ such that $x_i \oplus x_i'= -1$.
    
\end{enumerate}

In both use cases our privacy and utility guarantee depends on the Gaussian distribution. Given a random variable $X$, we denote by $X \sim N(\mu, \sigma^2)$ the fact that $X$ has Gaussian distribution with mean $\mu$ and variance $\sigma^2$ with the probability density function
\[
p_X(x) = \frac{1}{\sqrt{2 \pi \sigma}} e^{-\frac{(x-\mu)^2}{2\sigma^2}}.
\]
The multivariate Gaussian distribution is the multi-dimensional generalization of the Gaussian  distribution. For a random variable $X$, we denote by $X \sim N(\mu, \Sigma)$ the fact that $X$ has a multivariate Gaussian distribution with mean $\mu \in \real^d$ and covariance matrix $\Sigma \in \real^{d \times d}$ which is defined as $\Sigma = \E[(X-\mu)(X-\mu)^\top]$. The  probability density function of a multivariate Gaussian has a closed form formula:
\[
p_X(x) = \frac{1}{\sqrt{(2 \pi)^\streamlength \mathsf{det}(\Sigma)}} e^{-(x-\mu)^\top \Sigma^{-1} (x-\mu) },
\] 
where $\mathsf{det}(\Sigma)$ denotes the determinant of $\Sigma$. The covariance matrix is a positive definite matrix. 
We use the following fact regarding the multivariate Gaussian distribution:
\begin{fact}
\label{fact:multivariate}
Let $X \sim N(\mu, \Sigma)$ be a $d$-dimensional multivariate Gaussian distribution. If $A \in \complex^{\streamlength \times d}$, then the multivariate random variable $Y = AX$ is distributed as though $Y \sim N(A\mu, A\Sigma A^\top)$.
\end{fact}

Our algorithm for continual counting uses the Gaussian mechanism. To define it, we need to first define the notion of $\ell_2$-sensitivity. For a function $f : \mathcal X^n \to \R^d$  its {\em $\ell_2$-sensitivity} is defined as 
\begin{align}
\Delta f := \max_{\text{neighboring }X,X' \in \calX^n} \norm{f(X) - f(X')}_2.
\label{eq:ell_2sensitivity}    
\end{align}

\begin{definition}
[Gaussian mechanism]
\label{def:gaussianmechanism}
Let $f : \mathcal X^n \to \R^d$ be a function with $\ell_2$-sensitivity $\Delta f$. For a given $\epsilon,\delta \in (0,1)$  given $X \in \mathcal X^n$
the Gaussian mechanism $\mathcal M$
returns $\mathcal M(X) =  f(X) + e$, where $e \sim N(0,C_{\epsilon,\delta}^2 (\Delta f)^2 \I_d)$.
\end{definition}

\begin{theorem}
\label{thm:gaussian}
For a given $\epsilon,\delta \in (0,1)$ the Gaussian mechanism $\mathcal M$ satisfies $(\epsilon,\delta)$-differential privacy.
\end{theorem}
We will use the following result:
\begin{theorem}
[Theorem 2.1 in Denisov, McMahan, Rush, Smith, and Thakurta~\cite{mcmahan2022private}]
\label{thm:denisovadaptive}
Let $A \in \real^{n\times n}$ be a lower-triangular full-rank query matrix, and let $A = BC$ be any factorization with the following property: for any two neighboring streams of vectors $x, x' \in \real^{n}$ , we have $\norm{C(x-x')} \leq \zeta$. 
Let $z \sim N(0,\zeta^2 C_{\epsilon,\delta}^2)^{n}$ with $\zeta$
large enough so that $\calM(x)=Ax+Bz=B(Cx+z)$ satisfies $(\epsilon,\delta)$-DP in the nonadaptive continual release model. Then, $\calM$ satisfies the same DP guarantee (with the same parameters) even when the rows of the input sequence are chosen adaptively.
\end{theorem}

We use the result by Edmonds, Nikolov, and Ullman~\cite{edmonds2020power}. In particular, Edmonds, Nikolov, and Ullman~\cite[Section 5]{edmonds2020power} showed that the accuracy for linear queries (when expressed as a query matrix $A$) can be characterized using $\gamma_{\op F}(A)$. 
An {\em instance  independent mechanism} can be written as 
$${\calM}(x) = (Ax+z)$$ for a workload matrix $A$. 
It is called {\em instance independent}\footnote{In~\cite{bhaskara2012unconditional} this was called \emph{oblivious}.} as the noise function used does not depend on the input instance $x$.
As they use a somewhat different notation from ours we reprove their result to show that their result can be restated in our notation as follows:
\begin{theorem}
\label{thm:factorization}
\begin{itemize}
    \item Given a linear function $f(x) = Ax$ with matrix $A \in \real^{n \times d}$ and a factorization $A=LR$, then for $\epsilon, \delta \in (0,1)$
\[
\calM_{L,R}(A,x) = Ax + z, \quad \text{where} \quad z \sim N \paren{0, C_{\epsilon,\delta}^2 \norm{R}_{1 \to 2}^2 LL^\top}
\]
is $(\epsilon,\delta)$-differential private under the neighboring relation considered in this paper. 

\item For a workload matrix $A$ consisting of $\streamlength$ queries let $\calM$ be a $(\epsilon,\delta)$-differentially private instance-independent mechanism for $A$ with $\epsilon >0$ and $0 \le \delta \le 1$. Then
\[
\meansquared(\calM,A,\streamlength) \geq C_{\epsilon}^2 \frac{\gamma_F(A)^2}{\streamlength},
\]
where $C_\epsilon =\frac{1}{e^{2\epsilon}-1}$ is the constant in \cite{kasiviswanathan2010price}.

\item For a workload matrix $A$ consisting of $\streamlength$ queries let $\calM$ be a $(\epsilon,\delta)$-differentially private mechanism for $A$ with $\epsilon >0$ and $0 \le \delta \le \frac{1}{4 e^{\epsilon}}$. Then
\[
\meansquared(\calM,A,\streamlength) \geq C_{\epsilon}^2 \frac{\gamma_F(A)^2}{\streamlength},
\]
where $C_\epsilon =\frac{1}{e^{2\epsilon}-1}$ is the constant in \cite{kasiviswanathan2010price}.
\end{itemize}
\end{theorem}
\begin{proof}
Let $y \sim N \paren{0, C_{\epsilon,\delta}^2 \norm{R}_{1 \to 2}^2}$. Then by the properties of the multivariate Gaussian distribution (Fact~\ref{fact:multivariate}), it holds that $Ly \sim N \paren{0, C_{\epsilon,\delta}^2 \norm{R}_{1 \to 2}^2 LL^\top}$.
Thus, the mechanism $\cal M$ returning on input $x$ the value $L(Rx+y)$ with $y \sim N \paren{0, C_{\epsilon,\delta}^2 \norm{R}_{1 \to 2}^2}$ has the same distribution as  the mechanism returning $Ax+z$ with $z \sim N \paren{0, C_{\epsilon,\delta}^2 \norm{R}_{1 \to 2}^2 LL^\top}$.
In particular, both have the same privacy properties.  Thus, to simplify the notation we call the latter mechanism ${\cal M}$ as well.

{Next let us show $(\epsilon,\delta)$-differential privacy. Consider the mechanism $\cal M'$ that on input $x$ returns $f(x) = Rx + y$ with $y \sim N \paren{0, C_{\epsilon,\delta}^2 \norm{R}_{1 \to 2}^2}$.
For any neighboring databases  $x$ and $x'$ represented in the form of an $m$-dimensional vector and differing in the $i$-th coordinate note that $\norm{R(x - x')}_2 = \norm{Re_i}_2 \leq \norm{R}_{1 \to 2}$ and, thus,
$\Delta_2 f = \max_{\text{neighboring }x,x'} \norm{R(x - x')}_2 \le \norm{R}_{1 \to 2}$.
Thus, $(\epsilon,\delta)$-differential privacy of $\cal M'$ follows from~Theorem~\ref{thm:gaussian}.
Finally note that $\cal M$ only postprocesses the output of $\cal M'$ by multiplying with the matrix $L$, and, thus,  the postprocessing property of $(\epsilon,\delta)$-differential privacy imply that
$\cal M$ has the same privacy properties as $\cal M'$
Hence $\cal M$ is $(\epsilon,\delta)$-differential private.}

{For the lower bound on the mean squared error, we first state the relationship between the notation of Edmonds, Nikolov, and Ullman~\cite{edmonds2020power} and our notation and then show the result in the theorem. For the ease of presentation, we only show the translation for their lower bounds for instance independent mechanism -- the reduction from instance dependent (or regular) mechanisms to instance independent follows 
from~\cite{bhaskara2012unconditional}, which  for every $\epsilon' >0$ and $0 \le \delta' \le 1$
turns any $(\epsilon',\delta')$-differentially private instance dependent mechanism into a 
$(2\epsilon', 2 e^{\epsilon'} \delta')$-differentially private instance independent mechanism, without increasing its mean-squared error. Thus, a lower bound on the mean-squared error for $(\epsilon',\delta')$-differentially private instance independent mechanisms  with $\epsilon' > 0$ and $0 \le \delta' \le 1$ turns into a lower bound (of the same value) for ($\epsilon, \delta$)-differentially private instance dependent mechanisms with $\epsilon > 0$ and $0 \le \delta \le \frac{1}{2 e^{\epsilon}}$.}

Our first point of departure is the way the factorization norm is defined. For a query matrix $A \in \real^{n \times m}$, Edmonds, Nikolov, and Ullman~\cite{edmonds2020power} defined the following norm, which we denote by $\overline \gamma_{\op F}(A)$ to indicate that their definition is a normalized version of our definition in \cref{eq:gammanorm}.
:
\[
\overline \gamma_{\op F}(A) = \min \set{\frac{1}{\sqrt{n}} \norm{L}_{\op F} \norm{R}_{1 \to 2}: A = LR} = \frac{1}{\sqrt{n}} \gamma_{\op F}(A).
\]

The second point of departure is the waythe  workload matrix is defined. Edmonds, Nikolov, and Ullman~\cite{edmonds2020power} define and use in their mechanisms the normalized form of a workload matrix, i.e., for a workload matrix $A \in \real^{n \times m}$, they consider $\overline A = \frac{1}{n}A$. 
With this notation, Edmonds, Nikolov, and Ullman~\cite[Theorem 28]{edmonds2020power} studied instance independent mechanisms. An {\em instance  independent mechanism} can be written as 
$$\overline{\calM}(x) = \frac{1}{n}(Ax+z)$$ for a workload matrix $A$. 
It is called instance independent as the noise function used does not depend on the input instance $x$.

Finally, their definition for the mean-squared error is the square-root of the standard definition of mean-squared error (and the definition used in this paper) -- their choice of defining mean-squared error is so that they can compare it easily with the $\ell_\infty$ error. To differentiate the two, we use the notation $\overline \meansquared$ to denote their error metric. More formally, they define 
\[
\overline{\meansquared}(\overline{\calM},  {A}, n)  = \max_{ x \in \{0,1\}^\streamlength} \E_{\mathcal M} \sparen{ \frac{1}{\streamlength } \norm{\overline{\mathcal{M}}(x) - \overline{A} x}_2^2}^{1/2}.
\]
They showed that for such an instance independent mechanism and a suitable constant $C>0$
\[
\overline \meansquared(\overline{\calM},  A, n) \ge \frac{\overline{\gamma}_{\op F}(A)}{C \epsilon n} .
\]
Let $x$ be the input that maximizes the $\overline \meansquared(\overline{\calM},  {A}, n)$. Now, 
\begin{align*}
   \E\sparen{ \frac{1}{n} \norm{\overline \calM(x) - \frac{Ax}n}_2^2}^{1/2} 
   = \E\sparen{ \frac{1}{n} \norm{\frac{\calM(x)}n - \frac{Ax}n}_2^2}^{1/2}  
   = \frac{1}{n} \E\sparen{ \frac{1}{n} \norm{\calM(x) - Ax}_2^2}^{1/2}. 
\end{align*}
In other words,
\[
 \E\sparen{ \frac{1}{n} \norm{\calM(x) - Ax}_2^2}^{1/2}  \geq \frac{\overline{\gamma}_{\op F}(A)}{C \epsilon}
\]
Finally, since $\sqrt{n}\overline{\gamma}_{\op F}(A) = \gamma_{\op F}(A)$, we have 
\[
\E\sparen{ \frac{1}{n} \norm{\calM(x) - Ax}_2^2}^{1/2}  \ge \frac{{\gamma}_{\op F}(A)}{C \epsilon \sqrt{n}}
\]
or equivalently,
\[
\E\sparen{ \frac{1}{n} \norm{\calM(x) - Ax}_2^2}  \ge \frac{{\gamma}_{\op F}(A)^2}{C^2 \epsilon^2 n}.
\]

Unraveling the proof of Edmonds, Nikolov, and Ullman~\cite{edmonds2020power} and the proof in \cite{kasiviswanathan2010price}, on which it is based, we see that 
$C\epsilon =1/ C_\epsilon$ as required. This completes the proof of Theorem~\ref{thm:factorization}.
\end{proof}

\section{Semidefinite Program for $\gamma_{\op{F}}(.)$ Norm}
\label{sec:sdpgammanorm}

In this section, we characterize $\gamma_{\op{F}}(A)$ for any $A\in\complex^{n\times m}$ as a semidefinite program. To begin we can safely restrict our attention to a factorization $A=BC$ such that  $\norm{B}_{\op{F}} = \norm{C}_{1\rightarrow 2}$. In particular, this assumption can be made for an optimal factorization as well. To see why this holds, let us consider a factorization $A= B C$ such that 
\[
\frac{\norm{B}_{\op{F}}}{\norm{C}_{1\rightarrow 2}} = \alpha(B,C) \qquad \text{for} \qquad \alpha(B,C) \ne 1.
\]
Then we can have another factorization $A = B'C'$ where
\[
B' = \frac{B}{\sqrt{\alpha(B,C)}} \qquad \text{and} \qquad C' = \sqrt{\alpha(B,C)}C
\]
satisfying $\norm{B'}_{\op{F}} = \norm{C'}_{1\rightarrow 2}$. Hence,
\[
\gamma_{\op{F}}(A) = \inf\left\{\eta: \norm{B}_{\op{F}} = \norm{C}_{1\rightarrow 2} = \sqrt{\eta} \; \text{and} \; A=BC\right\}.
\]
For the remainder of this section, let $A\in\complex^{n\times m}$ and $X\in \herm{\complex^{n+m}}$ be a matrix written in the following block form
\begin{align}\label{eq:sdpAequiv}
X = 
\begin{pmatrix}
X_1 & X_2 \\
X_2^* & X_3
\end{pmatrix} 
\qquad \text{such that} \qquad X_2 = A.
\end{align}
For any factorization $A=BC$, the matrix $X$ as stated in \cref{eq:sdpAequiv} satisfies
\[
X\in\pos{\complex^{n+m}}
\qquad \text{if and only if} \qquad
X=WW^* 
\qquad \text{for} \qquad
W = \begin{pmatrix}
B \\ C^*
\end{pmatrix}.
\]
This implies that $X_1 = BB^*$, $X_2 = A = BC$, and $X_3 = C^*C$. Moreover, it is clear that 
\[
\trace(X_1) = \trace(BB^*) = \norm{B}_{\op{F}}^2
\qquad \text{and} \qquad
X_3[i,i] = \norm{C[;i]}_2^2.
\]
Let $\Phi:\herm{\complex^{n+m}} \rightarrow \herm{\complex^{n+m}}$ be the linear map defined as 
\begin{align}\label{eq:defPhi}
\Phi(X) = 
\widehat J_{n,m} \bullet X 
\quad \text{where} \quad
\widehat J_{n,m} =
\begin{pmatrix}
    0 & J_{n,m} \\
    J_{n,m}^* & 0
\end{pmatrix}
\qquad \text{and let} \qquad
\widehat A = 
\begin{pmatrix}
    0 & A \\
    A^* & 0
\end{pmatrix}.
\end{align}

\noindent Then the SDP for $\gamma_{\op{F}}(A)$ can be written as follows:
\begin{align*}
    \gamma_{\op{F}}(A):=        \min \quad & \eta \\
        \text{s.t.} \quad 
        & \Phi(X) = \widehat J_{n,m} \bullet X = \widehat A \\ 
        & \sum_{i\in [n]} X[i,i] \le \eta \\
        & X[i,i] \leq \eta \; \forall \; i \in [n+1, n+m] \\
        & X \in \pos{\complex^{n+m}}.
\end{align*}

\noindent We remark that strong duality holds for the above SDP and its associated dual (refer to Appendix~\ref{app:dualgammanorm} for the proof). For any optimal solution pair $(\eta, X)$, it is necessarily true that 

\begin{align}
\label{eq:SDPoptsolnconditon}
\sum\limits_{i\in [n]} X[i,i] = \eta.
\end{align}
Otherwise, we can construct a solution pair with optimal value strictly less than $\eta$. To see why this is true, let $\alpha = \sqrt{\eta/\eta'} >1$
where $$\sum\limits_{i\in [n]} X[i,i] = \eta' < \eta.$$ For
\[
X = 
\begin{pmatrix}
B \\ C^*
\end{pmatrix}
\begin{pmatrix}
B^\top & C
\end{pmatrix}
\qquad \text{let} \qquad
X' = 
\begin{pmatrix}
\sqrt{\alpha}B \\ \sqrt{\frac{1}\alpha}C^*
\end{pmatrix}
\begin{pmatrix}
\sqrt{\alpha}B^\top & \sqrt{\frac{1}\alpha} C
\end{pmatrix}.
\]
It is evident that $X'\in\pos{\complex^{n+m}}$ is a feasible solution of the aforementioned SDP since $\Phi(X') = \Phi(X) = \widehat A$. Moreover, 
\[
\sum_{i\in[n]} X'[i,i] = \alpha\trace(BB^*) = \alpha\eta' = \sqrt{\eta'\eta},
\]
and 
\[
X'[i,i] = \frac{X[i,i]}{\alpha} \leq \frac{\eta}{\alpha} = \sqrt{\eta'\eta} 
\qquad \text{for all} \qquad 
i\in[n+1,n+m].
\]
Hence the pair $(\sqrt{\eta'\eta}, X')$ forms a feasible solution, and since $\sqrt{\eta'\eta} < \eta$, it contradicts our assumption that $(\eta, X)$ is an optimal solution. This implies that \cref{eq:SDPoptsolnconditon} holds necessarily for any optimal solution. We now proceed to write a reformulation of the dual of the above SDP in the form that we make use of in all our lower bounds:
\begin{align}
\label{eq:dualgeneral2}
    \begin{split}
    \gamma_{\op{F}}(A)=    \max \quad & w^* (\widehat A \bullet \widehat X) w \\
        \text{s.t.} \quad & \begin{pmatrix} n\I_n & 0 \\ 0 &  \I_m \end{pmatrix} \succeq \widehat Z \\
        & w = \begin{pmatrix}
        w_1 \\ w_2
        \end{pmatrix} \quad \text{such that} \quad \norm{w}_2 = 1 
        \; \text{and} \; w_1 = \alpha1_n \\
        & \widehat Z \in \herm{\complex^{n+m}}, \; \alpha\in\real_{++},\; \text{and} \; w \in \real^{n+m}_{++}.
    \end{split}
\end{align}
We refer interested readers to Appendix~\ref{app:dualgammanorm} for an explanation of how we arrive at such a formulation. Note that the above form is reminiscent of a reformulation of $\norm{.}_{\op{cb}}$ norm due to Haagerup~\cite{haagerup1980decomposition}, but has a strictly smaller feasible set.

\section{Proof of \Cref{lem:gammanormlowerboundmain}}
\label{sec:proofmainlemma}
The lower bound in our main result require the following two propositions. 

\begin{prop}
\label{prop:blockpsdfordualmain}
Let $U\in\complex^{n\times p}$ and $V\in\complex^{m\times p}$ such that $\norm{U}_{\infty} \le 1$ and $\norm{V}_{\infty} \le 1$, where $\norm{.}_\infty$ denotes the spectral norm.  Then
\[
\begin{pmatrix}
\abs{\varrho_1}^2 \I_n & \varrho_1 \varrho_2 UV^* \\
\varrho_1^*\varrho_2^* VU^* & \abs{\varrho_2}^2 \I_m
\end{pmatrix}
\succeq 0 
\qquad \text{for all} \qquad \varrho_1, \varrho_2 \in \complex.
\]
\end{prop}
\begin{proof}
[Proof of \Cref{prop:blockpsdfordualmain}]
Given that $\norm{U}_{\infty} \le 1$ and $\norm{V}_{\infty} \le 1$, we have $UU^* \preceq \I_n$ and $VV^* \preceq \I_m$. It follows that
\[
\begin{pmatrix}
\abs{\varrho_1}^2 \I_n & \varrho_1\varrho_2 UV^* \\
\varrho_1^*\varrho_2^* VU^* & \abs{\varrho_2}^2 \I_m
\end{pmatrix} 
\quad \succeq \quad
\begin{pmatrix}
\abs{\varrho_1}^2 UU^* & \varrho_1\varrho_2 UV^* \\
\varrho_1^*\varrho_2^* VU^* & \abs{\varrho_2}^2 VV^*
\end{pmatrix} 
\quad = \quad
\begin{pmatrix}
\varrho_1 U \\ \varrho_2^* V
\end{pmatrix}
\begin{pmatrix}
\varrho_1 U \\ \varrho_2^* V
\end{pmatrix}^*
\succeq 0.
\]
This completes the proof of the proposition.
\end{proof}

\noindent Our lower bound  also used a  simple fact about block positive semidefinite matrices where the diagonal blocks are scalar multiple of identity matrices (\Cref{prop:blockpsdfordualmain}). We prove it next.
\begin{prop}
\label{prop:dotproductJABnonsquaremain}
Let $A, B \in \complex^{n\times m}$. Then $1_n^*(A\bullet B)1_m = \trace(A^*B)$, where $\trace(.)$ denotes the trace of the matrix.
\end{prop}
\begin{proof}
A simple calculation shows that
\[
1_n^*(A\bullet B)1_m = \trace((A\bullet B)1_m1_n^*) = \trace((A\bullet B)J_{m,n}) = \sum_{i=1}^n\sum_{j=1}^m A[i,j]B[i,j] = \trace(A^*B)
\]
completing the proof of Proposition~\ref{prop:dotproductJABnonsquaremain}.
\end{proof}

We now return to the proof of \Cref{lem:gammanormlowerboundmain}. 

\begin{proof}
[Proof of \Cref{lem:gammanormlowerboundmain}]
We construct dual variables $w$ and $\widehat Z$
for the SDP  in \Cref{fig:sdp}
that achieve the objective value as stated in the lemma. Let $A = U\Sigma_A V^*$ be the singular value decomposition of $A$. 
Let
\begin{align}
\label{eq:appdualchoicematrixvarmain}
w = \frac{1}{\sqrt{2}}
\begin{pmatrix}
1_n/\sqrt{n} \\ 1_m/\sqrt{m}
\end{pmatrix}
\qquad\qquad \text{and} \qquad\qquad
\widehat Z = 
\begin{pmatrix}
0^{n \times n} & Z \\
Z^* & 0^{m \times m}
\end{pmatrix},
\quad \text{where} \quad
Z = \sqrt{n}UV^* \in \complex^{n \times m}.
\end{align}
Since $\norm{U}_{\infty} = \norm{V}_{\infty} = 1$,  by setting $\varrho_1 = \sqrt{n}$ and $\varrho_2 = -1$ in Proposition~\ref{prop:blockpsdfordualmain}, we get that $\widehat Z$ is in the dual feasible set of the SDP defined in \Cref{fig:sdp}.  It is not hard to see that all the constraints imposed by the dual SDP on the vector $w$ is satisfied with 
$\alpha = 1/\sqrt{2n}$.
It remains to show the value of the objective function achieved by this dual solution. Using Proposition~\ref{prop:dotproductJABnonsquaremain}, the value of the objective function is
\[
w^*\left(\widehat A \bullet \widehat Z\right)w = 
\frac{1}{2\sqrt{nm}}\left(1_m^* (A^*\bullet Z^*)1_n + 1_n^* (A\bullet Z)1_m\right) =
\frac{1}{2\sqrt{nm}}\trace(AZ^* + A^*Z) = \frac{1}{\sqrt{m}}\trace(\Sigma_A), 
\]
where the last equality follows due to the following argument: since $AZ^* = \sqrt{n} U\Sigma_A U^*$ and $A^*Z = \sqrt{n} V\Sigma_A V^*$,  we have $\trace(AZ^* + A^*Z) = 2\sqrt{n}\trace(\Sigma_A)$ using the cyclic property of  a trace. 
Hence
\[
\gamma_{\op{F}}(A) \ge w^*\left(\widehat A \bullet \widehat Z\right)w =\frac{1}{\sqrt{m}}\trace(\Sigma_A) = \frac{\norm{A}_1}{\sqrt{m}}
\]
by the definition of the Schatten-1 norm. 

We now turn to proving the upper bounds. The upper bound can be obtained via constructing a factorization of a matrix. Fix a matrix $A\in\complex^{n\times m}$. For a factorization $A=LR$, let $L=A$ and $R=\I_m$. We have $\norm{L}_{\op{F}} = \norm{A}_{\op{F}}$ and $\norm{R}_{1\to 2} = 1$, and hence
\begin{align}\label{eq:gammaupperbound}
\gamma_{\op{F}}(A) \leq \norm{A}_{\op{F}}.
\end{align}
This completes the proof of \Cref{lem:gammanormlowerboundmain}.
\end{proof}
\section{Proof of the Bounds on Differentially Private Continual Counting}
\label{sec:boundsgammabound}
\label{sec:lowerbound2}

\subsection{Proof of Upper Bound on Differentially Private Continual Counting}
\label{sec:upperboundcounting}
As \Cref{alg:factorizationmechanism} shows, we have $L=R$. This is definitely a more restrictive setting and any upper bound under this restriction is also an upper bound on $\gamma_{\op{F}}(\counting)$. We will show that even under this restriction, we get an almost tight factorization and leave the question of finding an even tighter factorization when this restriction is removed as a direction of future research.

The requirement $L=R$ results in $\streamlength(\streamlength+1)/2$ equations in $\streamlength(\streamlength+1)/2$ variables. Our first observation is that the entries on any $t \times t$ principal submatrix of $L$ and $R$ are independent of the rest of the entries of $L$ and $R$; however, they define the rest of the entries. The second observation we make is that $L$ and $R$ are a Toeplitz matrix with a special structure: the principal diagonal entries all have to be the same and equal to $1$, and the $k$-th lower diagonal would be $\paren{1- \frac{1}{2(k-1)}}$ times the entries in $(k-1)$-th lower diagonal. In other words, we get the recurrence relation
\begin{align}
f(k)= 
 \begin{cases}
 1 & k = 0\\
 \left(1- \frac{1}{2k}\right) f(k-1) & k\geq 1\\
 \end{cases}
 \label{eq:f_k}
\end{align}
that  defines the entries of the factors as 
\begin{align*}
L = R = 
\begin{pmatrix}
f(0) &0 & \cdots & 0 \\
f(1) & f(0) & \cdots & 0 \\
\vdots & \vdots & \ddots & \vdots \\
f(\streamlength-2) & f(\streamlength-3) & \cdots & 0 \\
f(\streamlength-1) & f(\streamlength-2) & \cdots & f(0)
 \end{pmatrix}. 
\end{align*}

Note that this factorization is the same as in \Cref{alg:factorizationmechanism}. 
By construction, $\counting = LR$, so all that remains is to prove the bound on $\norm{L}_{\op F} \norm{R}_{1 \to 2}$ and the accuracy guarantee follows from \cref{eq:meansquarederror}.  The factorization into two Toeplitz matrices also means that we have bounded operators on the Hilbert space.  First, Theorem~\ref{thm:double_factorial} gives us for all $1 \leq t \leq \streamlength$,
\begin{align}
\norm{L[t;]}_{1 \to 2}^2 
    &=   \paren{ 1 + \sum_{i=1}^{t-1} \prod_{j=1}^i \paren{1 - \frac{1}{2j}}^2} 
    =  \paren{1 + \sum_{i=1}^{t-1} \paren{\prod_{j=1}^i  \paren{\frac{2j-1}{2j}} }^2} \nonumber  \\
    & =   \paren{1 + \sum_{i=1}^{t-1} \paren{\frac{(2i-1)!!}{(2i)!!}}^2}  
    \leq   \paren{1 + \frac{4}\pi  \sum_{i=1}^{t-1} \frac{1}{(4k+1)}}  \nonumber  \\
    &  \leq   \paren{1 + \frac{4}{\pi} \sparen{\frac{\ln(|4x+1|)}{4}}_{x=1}^{t-1}}
     \leq   \paren{1 +  \frac{1}{\pi}\ln \paren{\frac{4t-3}{5}} }
\label{eq:colR}    
\end{align}

Using \cref{eq:colR}, we therefore have 
\begin{align}
\begin{split}
    \| L \|_{\op{F}}^2 &= \sum_{t=1}^\streamlength \left\| L[t:]\right\|_2^2
    \leq   \sum_{t=1}^\streamlength \paren{1 +  \frac{1}{\pi}\ln \paren{\frac{4t-3}{5}} } \\ 
    &\leq   \paren{ \streamlength + \frac{\streamlength \ln(4\streamlength/5)}{\pi}  } =  {\streamlength} \paren{ 1 + \frac{\ln(4\streamlength/5)}{\pi}  }
\end{split}
\label{eq:frobL}
\end{align}
since natural-log is a monotonically increasing strictly concave function. 

As $L = R$, using \cref{eq:colR}, it follow that
\begin{align*}
\norm{R}_{1 \to 2}^2 = \norm{R[\streamlength;]}_2^2  \leq \paren{1 +  \frac{1}{\pi}\ln \paren{\frac{4\streamlength}{5}} }.
\end{align*}
Combining the two bounds, we have the upper bound in \cref{eq:gammanormboud}. \Cref{eq:mainupperbound} now follows using \cref{eq:meansquarederror}.

The privacy proof follows from the fact that our mechanism is an instantiation of the matrix mechanism. In particular, using Fact~\ref{fact:multivariate}, we can write 
\[
\counting x + z = L(Rx + y),
\]
where $y \sim N(0, \norm{R}_{1\to 2}^2 C_{\epsilon,\delta}^2 \I_\streamlength)$. Note that we can consider the multiplication with $L$ as a post-processing step as $L$ does not depend on $x$. Thus it suffices to argue that $f(x) = Rx$ is released in a differentially private manner. As stated in \Cref{def:gaussianmechanism}, the standard Gaussian mechanism for this problem releases $f(x) + y'$, where $y' \sim N(0, \Delta_2(f)^2 C_{\epsilon,\delta}^2 \I_\streamlength )$, where $\Delta_2(f)$ is the $\ell_2$-sensitivity of $f$. As  for two neighboring vectors $x$ and $x'$ that differ only in bit $i$ it holds that
\[
\norm{R(x-x')_2} = \norm{R e_i}_2 \le \norm{R}_{1 \to 2},
\]
it follows that the $\ell_2$-sensitivity of $f$ is $\norm{R}_{1 \to 2}$, which shows that $y$ was sampled from the appropriate normal distribution to preserve $(\epsilon,\delta)$-differential privacy of the complete execution over all round. The result for adaptivity follows from Theorem~\ref{thm:denisovadaptive}.

We can also improve the update-time of our algorithm as follows:
\begin{cor}
\label{cor:efficient_counting}
There is an efficient data-structure $\mathsf{D}$ and a continual counting mechanism $\calM$ that, for all $1\leq t \leq \streamlength$, on receiving a bit $x_t \in \set{0,1}$, outputs $a_t$ that satisfy $(\epsilon,\delta)$-differential privacy and 
\[
\meansquared(\calM, \counting, \streamlength) \leq C_{\epsilon,\delta}^2 \paren{1 + \frac{\ln(\streamlength)}{\pi}}^2 
\]
Further, the data structure $\mathsf{D}$ uses $O(\streamlength)$ space and uses $O(1)$ time per round, and pre-processing time of $O(\streamlength^2 + ns)$, where $s$ is the time required to sample from a normal distribution. 
\end{cor}
\begin{proof}
We present the non-adaptive continual counting algorithm; the adaptive continual counting result follows from Theorem~\ref{thm:denisovadaptive}. Let $\counting=LR$ be a factorization defined in the proof of Theorem~\ref{thm:counting}. The data structure $\mathsf D$ and the continual counting algorithm $\calM$ are defined as follows: During preprocessing we sample a vector $z \sim N(0,  \norm{R}_{1\to 2}^2 C_{\epsilon,\delta}^2 LL^\top)$, where $LR = \counting$ is the factorization computed in \Cref{sec:upperboundcounting}.
We describe below how to do this in $O(n^2 + ns)$ time.
For a stream $x = (x_1, \cdots, x_\streamlength) \in \set{0,1}^\streamlength$, our data structure $\mathsf D$ at time $t$  uses $O(n)$ space as it stores the following information:
\begin{enumerate}
    \item The current count $S_t = \sum_{i=1}^t x_i$.
    \item The vector $z \in \real^\streamlength$ sampled during preprocessing. 
\end{enumerate}
Our efficient continual counting mechanism $\calM$ consists of the following  steps:
For each round $t$, after receiving $x_t$, it simply outputs $c_t = S_t + z[t]$.

To prove the privacy guarantee note that the complete output $(c_1, \cdots, c_\streamlength)$ is
equal to the vector $\counting x + z$, since the input is chosen non-adaptively.
Using Fact~\ref{fact:multivariate}, we can write 
\[
\counting x + z = L(Rx + y),
\]
where $y \sim N(0, \norm{R}_{1\to 2}^2 C_{\epsilon,\delta}^2 \I_\streamlength)$. Note that we can consider the multiplication with $L$ as a post-processing step as $L$ does not depend on $x$. Thus it suffices to argue that $f(x) = Rx$ is released in a differentially private manner. As stated in \Cref{def:gaussianmechanism}, the standard Gaussian mechanism for this problem releases $f(x) + y'$, where $y' \sim N(0, \Delta_2(f)^2 C_{\epsilon,\delta}^2 \I_\streamlength )$, where $\Delta_2(f)$ is the $\ell_2$-sensitivity of $f$. As it holds for two neighboring vectors $x$ and $x'$ that differ only in bit $i$ that
\[
\norm{R(x-x')_2} = \norm{R e_i}_2 = \norm{R}_{1 \to 2},
\]
it follows that the $\ell_2$-sensitivity of $f$ is $\norm{R}_{1 \to 2}$, which shows that $y$ was sampled from the appropriate normal distribution. 

The analysis of the time per round  is straightforward. For the pre-processing time, note that, in general, sampling from a multivariate Gaussian $N(\mu, \Sigma)$ requires inverting the covariance matrix $\Sigma \succ 0$, which would require $O(n^3)$ time. However, in our case, we can sample a vector from the distribution $N(0, \norm{R}_{1 \to 2}^2 C_{\epsilon,\delta}^2 LL^\top)$ in time $O(n^2 +ns)$ by the following procedure:
\begin{enumerate}
    \item Sample $\streamlength$ Gaussian samples $(g_1,\cdots, g_\streamlength)$, i.i.d. from $N(0,1)$. This takes $O(ns)$ time, where $O(s)$ is the time required to sample from a normal distribution.
    \item Form a vector $g=C_{\epsilon,\delta}  \begin{pmatrix}
    g_1 & g_2 & \cdots & g_\streamlength
    \end{pmatrix}^\top$. This takes $n$ time.
    \item Output the vector $z = Lg$. This takes $O(n^2)$ time.
\end{enumerate}
By Fact~\ref{fact:multivariate} it follows that the vector $z$ has the same distribution as $N(0, LL^\top \norm{R}_{1\to 2}^2 C_{\epsilon,\delta}^2)$. 
This completes the proof of \Cref{cor:efficient_counting}.
\end{proof}


\subsection{Proof of Lower Bounds on Continual Counting}
\label{sec:lowerboundproofcounting}
\label{sec:lowerboundmain}
\begin{proof}
[Proof of Theorem \ref{thm:lowerboundgammanorm}]
Let $\set{\sigma_1(\counting), \sigma_2(\counting), \cdots, \sigma_\streamlength(\counting)}$ be the singular values of $\counting$. Note that $\counting$ is a non-singular matrix. We use the following well known fact that follows from noting that $(\counting^\top \counting)^{-1}$ is the matrix considered in \cite[Section 9]{strang2014functions} (also see Theorem~\ref{thm:matouvsek2020factorization}): 
\[
\sigma_i(\counting) = \frac{1}{2} \left\vert \csc \paren{\frac{(2i-1)\pi}{4\streamlength+2}} \right\vert \qquad \text{for all} \qquad i \in [n].
\]
Since $y^{-1} \leq \left\vert \csc(y) \right\vert$ for all $y > 0$ and Schatten-1 norm is just the sum of singular values, we have
\begin{align*}
    \norm{\counting}_1 &= \sum_{i=1}^\streamlength \sigma_i(\counting) = \frac{1}{2}\sum_{i=1}^\streamlength \left\vert \csc \paren{\frac{(2i-1)\pi}{4\streamlength+2}}\right\vert \geq \frac{2\streamlength+1}{\pi} \sum_{i=1}^\streamlength \frac{1}{2i-1} 
    = \frac{2\streamlength+1}{\pi} \paren{1 + \sum_{i=2}^{\streamlength} \frac{1}{2i-1}}. 
\\
& >  \frac{2\streamlength+1}{\pi} \paren{1 + \int\limits_{3}^{\streamlength+1} \frac{\mathsf dx}{2x-1} }  = \frac{2\streamlength+1}{\pi} \paren{1 + \frac{1}{2} \left(\ln(2\streamlength + 1) - \ln(5) \right) }.
\end{align*}
Setting $A = \counting$ and $m=\streamlength$ in \Cref{lem:gammanormlowerboundmain}, we therefore have 
\begin{align}
\begin{split}
 \gamma_{\op F}(\counting) 
& > \frac{\norm{\counting}_1}{\sqrt{\streamlength}} > \frac{2\streamlength+1}{\pi \sqrt{n}} \paren{1 + \frac{1}{2} \ln \left( \frac{2\streamlength + 1}5 \right) } 
\ge \frac{\sqrt n}{\pi} \fa.
\end{split}
\label{eq:lowerboundtightest}
\end{align}
Theorem \ref{thm:lowerboundgammanorm} follows by using \cref{eq:meansquaredgammanorm} and \cref{eq:lowerboundtightest}.
\end{proof}

\begin{proof}
[Proof of Theorem~\ref{thm:lowerboundadditive}]
We first prove the lower bound for instance independent mechanism. Using Theorem~\ref{thm:factorization}, we have that 
\begin{align*}
    \meansquared(\calM, \counting,\streamlength) \geq  \frac{\gamma_{\op F}(\counting)^2}{\streamlength (e^{2\epsilon}-1)^2}.
\end{align*}
for all instance-independent $(\epsilon,\delta)$-differentially private mechanism $\calM$. 
Now using \cref{eq:gammanormboud}, we have 
\begin{align*}
    \meansquared(\calM, \counting,\streamlength) \geq  \frac{1}{ (e^{2\epsilon}-1)^2 \pi^2}
    \fa^2
\end{align*}
for all instance-independent $(\epsilon,\delta)$-differentially private mechanism $\calM$.
This completes the second part of Theorem~\ref{thm:lowerboundadditive}. 

For instance-dependent mechanism, using the same proof as in Edmonds, Nikolov, and Ullman~\cite{edmonds2020power} and \cref{eq:gammanormboud}, we get 
\[
\meansquared(\calM, \counting,\streamlength) \geq \frac{1}{(e^{4\epsilon}-1)^2 \pi^2}
\fa
\]
This completes the proof of Theorem~\ref{thm:lowerboundadditive}.
\end{proof}



\section{Factorization View of Known Mechanisms and Suboptimality of the Binary Mechanism}
\label{sec:suboptimalbinarymechanism}
\label{sec:suboptimality}
In this section, we will first show how all known mechanisms for continual counting
can be seen as a matrix mechanism. Apart from the binary mechanism~\cite{Dwork-continual,chan2011private} there exist two variants by Honaker~\cite{honaker2015efficient}: one that is the optimized version and one that is suited for the continual observation.

\paragraph{Binary (Tree) Mechanism.}
Assume in the following that the stream length $\streamlength = 2^m$ for some $m \in \mathbb N$ and consider a complete binary tree with $\streamlength$ leaves and $2\streamlength-1$ nodes in total and let $\pi(i)$ denote  the path from leaf $i$ to the root. For $i\ge 1$, leaf $i$ is labeled with $x_i$, the $i$-th input in the streamed vector $x$. Each node with $2^k$ leaves in its subtree consists of the dyadic interval $[j 2^k, (j+1)2^k-1]$ of the input stream $x$ for some integer $j\ge 0$ and \emph{represents} the $k$-th bit in the binary representation of the leaves in its subtree.
The binary mechanism computes (a) the \emph{p-sum}
for each node in the binary tree consisting of the sum of the values of the leaves in its subtree with suitable noise, resulting in {\em noisy p-sum values},  and (b) then computes the $i$-th output by adding up the noisy p-sum values of the nodes on $\pi(i)$ that represent bits that are set to 1 in the binary representation of $i$. For example for $i=1$ only the p-sum of leaf $1$ is returned, for $i=2$ only the p-sum of the parent of leaves 1 and 2, for $i=3$ the sum of the p-sum of leaf 3 and of parent of leaves 1 and 2 is returned.

The following observation is straightforward: the binary mechanism computes a linear combination of the entries in the streamed vector $x$ as the p-sum value of each internal node of the binary tree is a linear combination of the entries of streamed vector $x$. 
Now we can consider the binary mechanism as a matrix mechanism. 
The right factor $R_{\mathsf{binary}}$ is constructed as follows: $R_{\mathsf{binary}} = W_m$, where  $W_0, \cdots, W_m$ are defined recursively as follows:
\begin{align*}
W_0 = \begin{pmatrix}
1
\end{pmatrix}, 
\quad W_k = \begin{pmatrix}
W_{k-1} & 0^{2^{k}-1 \times 2^{k-1}} \\
0^{2^{k}-1 \times 2^{k-1}} & W_{k-1} \\
(1_{2^{k-1}})^T & (1_{2^{k-1}})^T 
\end{pmatrix}, \quad k \leq m.
\end{align*}

Note that $R_{\mathsf{binary}} = W_m$ is a matrix in $\set{0,1}^{(2\streamlength -1) \times \streamlength }$, 
with each row corresponding to the p-sum computed for a node in the binary tree, where the ordering of the rows corresponds to a labeling of the nodes by post-order in the binary tree. Thus  a 1 in column $j$  of row $i$ indicates that $x_j$ contributes to the p-sum of the $i$-the node in post-order.
For example, the top-most row corresponds to the p-sum of the left-most leaf and the bottom-most row corresponds to the p-sum of the root of the binary tree.

The corresponding matrix $L_{\mathsf{binary}}$ is a matrix of
$\set{0,1}^{\streamlength \times (2\streamlength-1)}$,
where row $i$ 
contains a one in at most $\lceil \log_2 (i) \rceil$ entries.

\medskip

\paragraph{Honaker's optimization}
Honaker's optimization~\cite{honaker2015efficient} uses the same matrix as in the case of the  binary mechanism for the right matrix, i.e., $R_{\mathsf{honaker}} = R_{\mathsf{binary}}$. For the left matrix, he solves the optimization problem that minimizes the variance introduced. Even though it is not explicitly stated in~\cite{honaker2015efficient}, one can write the closed formula for his left matrix as $L_{\mathsf{honaker}}=  \counting R_{\mathsf{binary}}^{\dagger}$, where $R_{\mathsf{binary}}^{\dagger}$ denotes the Moore-Penrose pseudoinverse of $R_{\mathsf{binary}}$.

\medskip
\paragraph{Honaker's streaming version}
Honaker's streaming algorithm has the same right matrix as the  binary mechanism and his optimized mechanism. However, to ensure computation in the streaming model, Honaker's left matrix has to be constrained -- the left matrix is a constraint Moore-Penrose pseudoinverse. In particular, it does not have a closed-form expression, but can be computed in polynomial time using known algorithms from numerical analysis.

\medskip
\paragraph{Suboptimality of the Binary Mechanism.} We next give the detail proof of the suboptimality of the binary mechanism. 
 \begin{proof}
 [Proof of Theorem~\ref{thm:binarymechanismsuboptimal}]
Next we use our factorization view of  the binary mechanism to show its suboptimality as far as the constants are concerned.  It is easy to see that $\norm{R_{\mathsf{binary}}}_{1 \to 2}^2 = 1 + {\log_2(n)}$.
   Recall that we want to bound $\norm{L_{\mathsf{binary}}}_{\op F}$, which is the square root of the number of entries of $L$ that are 1.
Thus, a simple counting argument then yields 
\begin{align*}
\norm{L_{\mathsf{binary}}}_F^2 = \frac{\streamlength \log_2(\streamlength)}{2}.  
\end{align*}
In other words, we have argued the following:
\begin{lem}
The binary mechanism (defined in Dwork, Naor, Pittasi, and Rothblum~\cite{Dwork-continual} and Chan, Shi, and Song~\cite{chan2011private} can be represented as a factorization of $\counting$ into sparse matrices $L_{\mathsf{binary}}$ and $R_{\mathsf{binary}}$ such that
\[
\counting = L_{\mathsf{binary}} R_{\mathsf{binary}} \quad \text{and} \quad \norm{L_{\mathsf{binary}}}_F^2 \norm{R_{\mathsf{binary}}}_{1 \to 2}^2 = { \frac{\streamlength\log_2(\streamlength)}2 \paren{1  +  \log_2(\streamlength)}}.
\]
\end{lem}
Since, for any factorization using the Gaussian mechanism, we have that the expected mean squared error is exactly $C_{\epsilon,\delta} \norm{L}_F \norm{R}_{1 \to 2}$, we get that the binary tree with the Gaussian mechanism achieves suboptimal accuracy in comparison to our mechanism.
More precisely, the  mean squared error of the binary mechanism  is approximately a factor of 
$\frac{\pi^2}{2 (\ln 2)^2} \approx 10.2$ larger than that of our mechanism.

The second part of Theorem~\ref{thm:binarymechanismsuboptimal} follows from our lower bound on $\gamma_{\op F}(\counting)$ in \cref{eq:gammanormboud}.
 \end{proof}


\section{Non-asymptotic Bound on Private Online Optimization}
\label{sec:applications}

\label{sec:privopt}
In this section we prove  Theorem~\ref{thm:privopt}. For this
we modify the  algorithm for private online convex optimization, given in \Cref{alg:dp-ftrl} below
by replacing the binary mechanism in Line 6 by our mechanism. Our bounds 
assumes that the loss function $\ell:\mathcal K \times \mathcal D \to \real$ function  is convex and  $\kappa$-Lipschitz  with respect to the $\ell_2$ norm, that is, for all $d \in \mathcal D, \theta \in \mathcal K$, $$\norm{\nabla \ell(\theta;d)}_2 \leq \kappa.$$

We guarantee the Lipschitz property using the standard clipping method used in private learning. Let $\mathcal B_d(0,\kappa)$ denotes the $d$-dimensional Euclidean ball of radius $\kappa$ centered at origin. Then we define the  clipping function, $\mathsf{clip}:\real^d \times \real \to \mathcal B_d(0,\kappa)$ as follows:
\begin{align}
\label{eq:clipping}
\mathsf{clip}(g;\kappa) :=  \min \set{\frac{\kappa}{\norm{\nabla \ell(\theta_t;x_t)}^2},1 } \cdot g,
\end{align}

Note that $\mathsf{clip}(\nabla \ell(\theta_t;x_t);\kappa) \in \mathcal B_d(0,\kappa)$.

We first state our result on the regret bound (\Cref{def:regret}) on online private optimization in the 
 {\em adversarial regret model}~\cite{hazan2019introduction}, where the data
sample $x_t$ are drawn adversarially based on the past outputs $\set{\theta_1, \cdots, \theta_{t-1}}$. The result can be extended to {\em stochastic regret}~\cite{hazan2019introduction}, where the data is sampled i.i.d. from some fixed unknown distribution $\calD$ using standard techniques~\cite{hazan2019introduction}.

\begin{theorem}
[Restatement of Theorem~\ref{thm:privopt}]
Let $[\theta_1, \cdots, \theta_\streamlength]$ be the $d$-dimensional outputs of Algorithm $\onlineAlgorithm$, and $L$ be a bound on the $\ell_2$-Lipschitz constant of the loss functions. Then the following is true for any $\thetaopt \in \mathcal K$:
\[
\regret{\onlineAlgorithm}{n} \leq  \norm{\thetaopt}_2 \sqrt{\frac{\paren{1 + \frac{\ln(4n/5)}{\pi}}(\kappa^2 + \kappa C_{\epsilon,\delta}  \sqrt{d})} {2\streamlength}}.
\]
\end{theorem}
\begin{proof}
The privacy proof follows from the privacy of the Gaussian mechanism and noting that the $\ell_2$-sensitivity of $R(t)G(t)$ is at most $\kappa \norm{R(t)}_{1 \to 2}^2$ (as every row of $G(t)$ has $\ell_2$-norm norm at most $\kappa$ by clipping). 
The utility proof follows from the idea in Kairouz, McMahan, Song, Thakkar, Thakurta, and Xu~\cite{kairouz2021practical} using
$\rho(\theta) = \frac{\lambda}{2}\norm{\theta}_2^2$ (for a suitably chosen $\lambda > 0$) except that we replace their binary mechanism with our mechanism. We show that the proof also goes through with our mechanism. 

\begin{algorithm}[t]

\caption{Differentially private Follow-the-regularized leader, $\onlineAlgorithm$}
\begin{algorithmic}[1]
   \Require Dataset $D = (x_1, \cdots, x_\streamlength)$, dimension $d$ of the parameter space, constraint set $\mathcal K$, regularization parameter $\lambda$, clipping norm $\kappa$, $\ell_2^2$ regularizer $\rho(\theta) = \frac{1}{\lambda} \norm{\theta}_2^2$.
   \State Set $\theta_1 = \argmin_{\theta} \frac{\lambda}{2}\norm{\theta}_2^2$ represented as a column vector.
   \For {$t=1$ to $\streamlength$}
        \State $\nabla_t \gets \mathsf{clip}(\nabla \ell(\theta_t;x_t);\kappa)$, where $\mathsf{clip}(\cdot; \cdot)$ is as defined in \cref{eq:clipping}.  
        \State Define $G(t) = \begin{pmatrix}
        \nabla_1 & \nabla_2 & \cdots & \nabla_t
        \end{pmatrix}^\top$ whose $i$-th row is formed  by the row vector $\nabla_i^\top$.
        \State Sample $Z \in \real^{t \times d}$, where $Z[i,j] \sim N(0, C_{\epsilon, \delta}^2 \kappa^2 \norm{R(t)}_{1 \to 2}^2)$.
        \State Compute $s_t = L[t;](R(t) G(t) + Z)$, where $R(t)$ is the $t \times t$ principal submatrix of $R$.
        \State Update
        \[
        \theta_{t+1} = \argmin_\theta \paren{ \brak{s_t,\theta} + \frac{\lambda}{2}\norm{\theta}_2^2 }
        \]
        \State \textbf{Output} $\theta_{t+1}$.
    \EndFor
\end{algorithmic}
\label{alg:dp-ftrl}
\end{algorithm}

Let $\nabla_i$  denote the gradient of the current cost function at the current point $\nabla \ell(\theta_i;x_i)$ and 
\[
\thetaopt := \argmin_{\theta \in \mathcal K} \sum_{i=1}^\streamlength \ell(\theta;x_i).
\] 
Define the following optimizers for the non-private and private variant of follow-the-regularized leader, respectively:
\begin{align}
\begin{split}
    \widetilde \theta_{t+1} &= \argmin_{\theta \in \mathcal C} \sum_{i=1}^t \brak{\nabla_i, \theta} + \frac{\lambda}{2} \norm{\theta}_2^2 \\
    \theta_{t+1} &= \argmin_\theta \paren{ \brak{s_t,\theta} + \frac{\lambda}{2}\norm{\theta}_2^2 } \\
    &= \argmin_{\theta \in \mathcal C} \sum_{i=1}^t \brak{\nabla_i, \theta} + \frac{\lambda}{2} \norm{\theta}_2^2 + \brak{s_t - \sum_{i=1}^t \nabla_i, \theta},
\end{split}
\label{eq:thetaopt}
\end{align}
for $s_t$ be the estimate returned from the partial sum using our mechanism. 
Therefore, we have
\begin{align}
\begin{split}
\sum_{t=1}^\streamlength \ell(\theta_t;x_t) -  \sum_{t=1}^\streamlength \ell(\thetaopt;x_t) 
    & \overset{(*)}{\leq} \sum_{t=1}^\streamlength \brak{\nabla_t, \theta_t - \thetaopt} \\
    & = \underbrace{\sum_{t=1}^\streamlength \brak{\nabla_t, \theta_t - \widetilde \theta_t}}_{P} + \underbrace{\sum_{t=1}^\streamlength \brak{\nabla_t,  \widetilde \theta_t - \thetaopt}}_N,
\end{split}    
\label{eq:regretbound}
\end{align}
where (*) follows from \cite[eq. 5]{hazan2019introduction}.

We can now bound the term $N$ and $P$ separately. Note that $N$ is the regret if we did not had any privacy constraints (i.e., $\epsilon = \infty$). Therefore, we bound that term using the non-private regret bound of  follow-the-regularized leader with regularization function $\rho(\theta) = \frac{\norm{\theta}_2^2}{2\lambda}$. In other words, we set $\rho(\theta) = \frac{\norm{\theta}_2^2}{2\lambda}$, $u=\thetaopt$, and $\eta =1$ in Theorem~\ref{thm:hazan2019introduction}. Then using the fact that $\ell$ is $\kappa$-Lipschitz with respect to $\ell_2$ norm, we have 
\begin{align}
N \leq \frac{\streamlength \kappa^2}{\lambda} + \frac{\lambda}{2} \paren{\norm{\thetaopt}_2^2 - \norm{\theta_1}_2^2}.
    \label{eq:nonprivateterm}
\end{align}
For the term $P$, define 
\[
\phi_1(\theta) := \frac{1}{\lambda} \sum_{i=1}^t \paren{\brak{\nabla_i, \theta} + \frac{\lambda}{2} \norm{\theta}_2^2}, 
\qquad 
\phi_2(\theta):= \frac{1}{\lambda}  \paren{\sum_{i=1}^t \brak{\nabla_i, \theta}} + \frac{\lambda}{2} \norm{\theta}_2^2 + \brak{s_t - \sum_{i=1}^t \nabla_i, \theta}.
\]
Note that $\tilde \theta_t = \argmin_\theta \phi_1(\theta)$ and 
that $\theta_t = \argmin_\theta \phi_2(\theta)$, and  $$\psi(\theta) = \frac{1}{\lambda}\brak{s_t - \sum_{i=1}^t \nabla_i, \theta}$$ 
 is a linear function with 
$$\nabla \psi(\theta) = \frac{1}{\lambda} s_t - \sum_{j=1}^t \nabla_j$$ being the subgradient of $\psi$ at $\theta \in \mathcal K$. Furthermore, $\phi_1(\theta)$ and $\phi_2(\theta)$ are  quadratic functions and 
for $\theta \in \mathcal K$
\[
\nabla \phi_1(\theta) = \lambda \theta  + \frac{1}{\lambda} \sum_{i=1}^t \nabla_i
\quad \text{and} \quad 
\nabla^2 \phi_1(\theta) = \lambda \I_d.
\]
and
\[
\nabla \phi_2(\theta) = \lambda \theta  + \frac{1}{\lambda} \sum_{i=1}^t \nabla_i + s_t - \sum_{i=1}^t \nabla_i
\quad \text{and} \quad 
\nabla^2 \phi_2(\theta) = \lambda \I_d.
\]

It follows from Fact~\ref{fact:strongconvexity} that  $\phi_2$ is $\lambda$-strongly convex  and that $\psi$ and $\phi_2$ are convex.
We can thus apply \Cref{lem:mcmahan}
to get
\[
\norm{  \widetilde \theta_t - \theta_t}_2 \leq \frac{1}{\lambda} \norm{ L[t;] Z_t}_2 ,
\]
where $L[t;]$ is the $t$-dimensional $t$-th row of $L$ and $Z_t \in \real^{t \times d}$ is a random Gaussian matrix such that $Z_t[k,j] \sim N(0, C_{\epsilon,\delta}^2 \norm{R(t)}_{1 \to 2}^2)$. The result of this product is a $d$-dimensional vector. 

Since $\ell(\cdot; \cdot)$ is $\kappa$-Lipschitz with respect to $\ell_2$ norm in its first parameter, we have
\begin{align}
P \leq \sum_{t=1}^\streamlength \norm{\nabla_t}_2 \norm{  \widetilde \theta_t -  \theta_t}_2 \leq \kappa \sum_{t=1}^\streamlength \norm{  \widetilde \theta_t - \theta_t}_2,
\label{eq:privateterm}    
\end{align}
where the first inequality follows from Cauchy-Schwarz inequality.

Taking expectation in \cref{eq:regretbound} and substituting the bounds on $P$ and $N$ in \cref{eq:privateterm} and \cref{eq:nonprivateterm}, respectively, we get
\begin{align}
    \frac{1}{\streamlength}\E \sparen{\sum_{t=1}^\streamlength \ell(\theta_t;x_t) -  \sum_{t=1}^\streamlength \ell(\thetaopt;x_t) } 
        &\leq \frac{\kappa^2}{\lambda} + \frac{\lambda}{2\streamlength} \paren{\norm{\thetaopt}_2^2 - \norm{\theta_1}_2^2} + \frac{\kappa}{ \streamlength} \sum_{t=1}^\streamlength \E\sparen{ \norm{  \widetilde \theta_t - \theta_t}_2 }
\end{align}
Since every entry of $Z_t[k,j] \sim N(0, C_{\epsilon,\delta}^2 \norm{R(t)}^2_{1 \to 2})$ for $1 \leq k \leq t, 1 \leq j \leq d$, we have  by the independence of the $Z_t[k,j]$ variables that 
\begin{align*}
\E \sparen{\norm{L[t;]Z_t}_2^2} &= \E \sparen{ \sum_{j=1}^d (L[t;] Z_t[;j])^2 } = \sum_{j=1}^d \Var{L[t:]Z_t[:j]} 
=\sum_{j=1}^d \sum_{k=1}^t \Var{L[t,k] Z_t[k,j]}  \\
    &\leq d \ \sum_{k=1}^t {L[t,k]}^2 C_{\epsilon,\delta}^2  \norm{R(t)}_{1 \to 2}^2  =  d\  C_{\epsilon,\delta}^2  \norm{R(t)}_{1 \to 2}^2 \norm{L[t;]}_2^2
\end{align*}
Thus, it follow that
\[
\E \sparen{\norm{L[t;]Z_t}_2} \leq \sqrt{\E  \sparen{\norm{L[t;]Z_t}_2^2}} \le  C_{\epsilon,\delta} \sqrt{d} \norm{R(t)}_{1 \to 2} \norm{L_t[t;]}_2.
\]

\noindent Combining all of this, we get
\begin{align}
    \begin{split}
        \frac{1}{\streamlength}\E \sparen{\sum_{t=1}^\streamlength \ell(\theta_t;x_t) -  \sum_{t=1}^\streamlength \ell(\thetaopt;x_t) } 
        &\leq \frac{\kappa^2}{\lambda} + \frac{\lambda}{2\streamlength} \paren{\norm{\thetaopt}_2^2 - \norm{\theta_1}_2^2} + \frac{\kappa}{ \streamlength} \sum_{t=1}^\streamlength \E\sparen{ \norm{  \widetilde \theta_t - \theta_t}_2 } \\
        &\leq \frac{\kappa^2}{\lambda} + \frac{\lambda}{2\streamlength} \paren{\norm{\thetaopt}_2^2 - \norm{\theta_1}_2^2} + \frac{\kappa}{\lambda \streamlength} \sum_{t=1}^\streamlength \E\sparen{  \norm{ L[t;] Z_t}_2 } \\
        &\leq \frac{\kappa^2}{\lambda} + \frac{\lambda}{2\streamlength} \norm{\thetaopt}_2^2  + \frac{\kappa\sqrt{d}}{\streamlength \lambda} C_{\epsilon,\delta}  \sum_{t=1}^\streamlength \paren{1 + \frac{\ln(4t/5)}{\pi}} \\
        &\leq \frac{\kappa^2}{\lambda} + \frac{\lambda}{2\streamlength} \norm{\thetaopt}_2^2  + \frac{\kappa\sqrt{d}}{\lambda} C_{\epsilon,\delta}  \paren{1 + \frac{\ln(4n/5)}{\pi}}.
    \end{split}
\end{align}

\noindent We now optimize for $\lambda$ to get 
$$\lambda =  \frac{\sqrt{2\streamlength\paren{1 + \frac{\ln(4\streamlength/5)}{\pi}} (\kappa^2 + \kappa C_{\epsilon,\delta}  \sqrt{d})} }{\norm{\thetaopt}_2},$$
and consequently, the bound on regret is
\[
\regret{\onlineAlgorithm}{n} \leq  \norm{\thetaopt}_2 \sqrt{\frac{\paren{1 + \frac{\ln(4\streamlength/5)}{\pi}}(\kappa^2 + \kappa C_{\epsilon,\delta}  \sqrt{d})} {2\streamlength}}
\]
completing the proof of Theorem~\ref{thm:privopt}.
\end{proof}

\section{Lower Bound on Parity Queries} 
 \label{sec:paritylowerbound}

We follow the same approach as for $\counting$. We use the observation of Edmonds, Nikolov, and Ullman~\cite{edmonds2020power} that the query matrix corresponding to any set of the parity queries is the  $\binom{d}{w}$ matrix formed by taking the corresponding rows of the $2^d \times 2^d$ unnormalized Hadamard matrix. Let us call this matrix $S$. 

We now set $A = S, n=\binom{d}{w}$ and $m = 2^{d}$ in \Cref{lem:gammanormlowerboundmain}. Using the fact that the singular value of $S$ is $2^{d/2}$ with multiplicity $\binom{d}{w}$, we get $\norm{S}_1 = 2^{d/2}\binom{d}{w}$ and a lower bound on $\gamma_{\op F}(S) \geq \binom{d}{w}$. Theorem~\ref{thm:parity} now follows by an application of Theorem~\ref{thm:factorization}.

\subsection*{Acknowledgements.} 
 This project has received funding from the European Research Council (ERC)
  under the European Union's Horizon 2020 research and innovation programme (Grant agreement 
   \begin{wrapfigure}{r}{0.15\textwidth}\label{fig:diff}
  \includegraphics[width=0.13\textwidth]{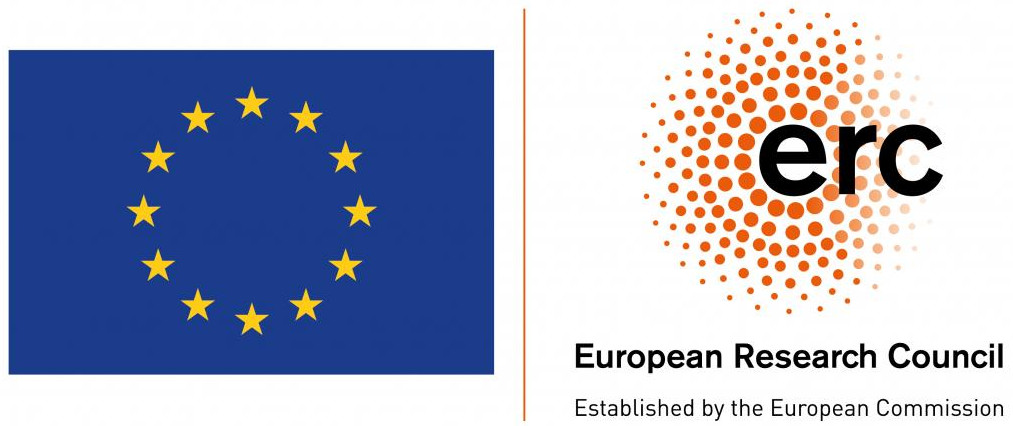}
\end{wrapfigure}
 No. 101019564 ``The Design of Modern Fully Dynamic Data Structures (MoDynStruct)'' and from the Austrian Science Fund (FWF) project ``Fast Algorithms for a Reactive Network Layer (ReactNet)'', P~33775-N, with additional funding from the \textit{netidee SCIENCE Stiftung}, 2020--2024. JU's research was supported by Decanal Research grant. The authors would like to thank Aleksandar Nikolov for the useful discussions on factorization norms, and Rasmus Pagh and Amir Yehudayoff for pointing us to the paper by Bennett. 

\newcommand{\etalchar}[1]{$^{#1}$}

\appendix

\section{Dual Characterization of $\gamma_{\op{F}}(A)$ SDP}\label{app:dualgammanorm}
Recall that the (primal problem of) SDP for $\gamma_{\op{F}}(A)$ can be written as follows:
\begin{align}
\label{eq:appprimalgeneral2}
    \begin{split}
        \gamma_{\op{F}}(A):=        \inf \quad & \eta \\
            \text{s.t.} \quad 
            & \Phi(X) = \widehat J_{n,m} \bullet X = \widehat A \\ 
            & \sum_{i\in [n]} X[i,i] \le \eta \\
            & X[i,i] \leq \eta \; \forall \; i \in [n+1, n+m] \\
            & X \in \pos{\complex^{n+m}}.
    \end{split}
\end{align}
We now proceed to characterize the dual of the above SDP via the Lagrangian dual method. To this end, we turn the SDP into an equivalent unconstrained optimization problem by introducing appropriate penalty terms for each constraints. Specifically, let $W\in\herm{\complex^{n+m}}$, $S\in\pos{\complex^{n+m}}$, $\beta \in \real_+$, and $y \in \real^m_+$ be the penalty terms. The following optimization problem is equivalent to solving the SDP described in \cref{eq:appprimalgeneral2}:
\[
\sup_{\beta, y, W, S} \; \inf_{\eta, X}
\left\{
\ip{\eta}{1} + \ip{\widehat A-\Phi(X)}{W} + \ip{\sum_{i\in[n]}X[i,i]-\eta}{\beta} +
\sum_{j\in[m]}\ip{X[n+j,n+j]-\eta}{y[j]} - \ip{X}{S}
\right\}.
\]
We need to note that for any feasible solution pair $(\eta, X)$ of \cref{eq:appprimalgeneral2}, the best choice of penalty variables can only achieve the value $\eta$. If a candidate solution pair $(\eta, X)$ violates any constraint, then the corresponding penalty term can be chosen appropriately to drive the optimal value to an arbitrarily large quantity. Using the minimax inequality, we have that 
\[
\inf_{\eta, X} \; \sup_{\beta, y, W, S}  
\left\{
\ip{\eta}{1} + \ip{\widehat A-\Phi(X)}{W} + \ip{\sum_{i\in[n]}X[i,i]-\eta}{\beta} +
\sum_{j\in[m]}\ip{X[n+j,n+j]-\eta}{y[j]} - \ip{X}{S}
\right\}
\]
is at most $\gamma_{\op{F}}(A)$ for $X\in\pos{\complex^{n+m}}$. Rearranging the terms, we can rewrite the aforementioned optimization problem as
\begin{align}\label{eq:lagrange}
\inf_{\eta, X} \; \sup_{\beta, y, W, S}
\left\{
\ip{\widehat A}{Y} + \ip{X}{\begin{pmatrix} \beta\I_n & 0 \\ 0 & \Delta_m(y)\end{pmatrix}-S-\Phi^*(W)} + \ip{\eta}{1-\beta-\sum_{j\in[m]}y[j]}
\right\}
\end{align}

Here (as mentioned in Section~\ref{sec:linal}) $\Delta_m:\complex^m\rightarrow \complex^{m\times m}$ is the linear map that maps a vector to a diagonal matrix and $\Phi^*:\complex^{(n+m) \times (n+m)}\rightarrow \complex^{(n+m) \times (n+m)}$ is the unique linear map (called the {\em adjoint map} of $\Phi$ stated in \cref{eq:defPhi}) such that
\[
\ip{\Phi(X)}{W} = \ip{X}{\Phi^*(W)} \;\qquad \text{for all} \;\qquad X,W \in \complex^{(n+m) \times (n+m)}.
\]

\noindent Note that the above relationship between $\Phi$ and $\Phi^*$ is for all compatible matrices (not necessarily Hermitian) $X$ and $Y$ of appropriate dimension. We now turn to defining $\Phi^*$. For any $X, W \in \complex^{(n+m) \times (n+m)}$
\begin{align*}
\ip{\Phi(X)}{W}
& =
\ip{\Phi\begin{pmatrix}
X_{11} & X_{12} \\
X_{21} & X_{22}
\end{pmatrix}}
{\begin{pmatrix}
W_{11} & W_{12} \\
W_{21} & W_{22}
\end{pmatrix}} 
= 
\ip{\begin{pmatrix}
0 & X_{12} \\
X_{21} & 0
\end{pmatrix}}
{\begin{pmatrix}
W_{11} & W_{12} \\
W_{21} & W_{22}
\end{pmatrix}} \\
& = 
\ip{\begin{pmatrix}
X_{11} & X_{12} \\
X_{21} & X_{22}
\end{pmatrix}}
{\begin{pmatrix}
0 & W_{12} \\
W_{21} & 0
\end{pmatrix}} 
 =
\ip{X}{\Phi^*(W)}.
\end{align*}

\noindent We make a remark that for the construction of dual from \cref{eq:lagrange}, the dual variable $W$ is Hermitian and hence $W_{12} = W_{21}^*$. Now expression~\ref{eq:lagrange} can be viewed as an unconstrained version of a constrained optimization problem where the pair $(\eta,X)$ is the penalty term. Given that $\eta\in\real$ and $X\in\pos{\complex^{n+m}}$, we have the following constraints:
\[
\begin{pmatrix} \beta\I_n & 0 \\ 0 & \Delta_m(y)\end{pmatrix} \succeq S + \Phi^*(W)
\qquad \text{and} \qquad
\beta + \sum_{j\in[m]}y[j] = 1.
\]

\noindent This can be rewritten as the following constrained optimization problem, which is the dual of \cref{eq:appprimalgeneral2}.
\begin{align}
    \label{eq:appdualgeneral2}
    \begin{split}
    \gamma_{\op{F}}(A) \ge        \sup \quad & \ip{\widehat A}{W} \\
        \text{s.t.} \quad 
        & \begin{pmatrix}
            \beta\I_n & 0 \\
            0 & \Delta_m(y)
          \end{pmatrix} \succeq \Phi^*(W) \\ 
        & \beta + \sum_{i\in [m]} y[i] = 1 \\
        & W \in \herm{\complex^{n+m}}, \; \beta\in\real_+, \; \text{and} \; y \in \real^m_+.
        \end{split}
\end{align}

\noindent We next show that strong duality holds for eqs~\ref{eq:appprimalgeneral2} and~\ref{eq:appdualgeneral2}.

\begin{lem}
\label{lem:slater}
The Slater condition for strong duality holds for both the primal and the dual problem described in eqs~\ref{eq:appprimalgeneral2} and~\ref{eq:appdualgeneral2}, respectively. In particular, there exist primal and dual feasible solutions that achieve the value $\gamma_{\op{F}}(A)$.
\end{lem}
\begin{proof}
The Slater condition for the primal problem asks for showing a primal feasible solution $\eta\in\real$ and $X\in\pd{\complex^{n+m}}$ such that
\[
\Phi(X) = \widehat A \qquad \text{and} \qquad \sum_{i\in[n]}X[i,i] < \eta \qquad \text{and} \qquad X[i,i] < \eta \; \forall \; i\in[n+1, n+m],
\]
i~e.~all constraints are satisfied and the inequalities are satisfied with strict inequalities.

If these conditions are met, then strong duality holds and there exists a dual feasible solution that achieves the optimum value $\gamma_{\op{F}}(A)$. To this end, let
\[
X = \left(\norm{A}_{\infty}+1\right)\I_{n+m} + \widehat A
\qquad \text{and} \qquad 
\eta = 2nm\left(\norm{A}_{\infty}+1\right).
\]
By construction, $\Phi(X) = \widehat{A}$. It is evident that
\[
\sum_{i\in[n]}X[i,i] = n\left(\norm{A}_{\infty}+1\right) < \eta
\qquad \text{and} \qquad
X[i,i] = \left(\norm{A}_{\infty}+1\right) < \eta
\; \forall \; i\in[n+1, n+m].
\]
Moreover, using Lemma~\ref{lem:schurcomplement},
we can show that $X\in\pd{\complex^{n+m}}$. Moving on to the dual, the Slater condition for it asks for constructing a dual feasible solution $W\in \herm{\complex^{n+m}}$, $\beta \in \real_{++}$, and $y \in \real^m_{++}$ such that
\begin{align}\label{eq:appstrictdualfeasibility}
\begin{pmatrix}
\beta\I_n & 0 \\
0 & \Delta_m(y)
\end{pmatrix} \succ \Phi^*(W)
\qquad \text{and} \qquad
\beta + \sum_{i\in[m]}y[i] = 1.
\end{align}
If these conditions are met, then strong duality holds and there exists a primal feasible solution that achieves the optimum value $\gamma_{\op{F}}(A)$. 
Set 
\[
\beta = \frac{1}{2}
\qquad \text{and} \qquad 
y = \frac{1_m}{2m} \in \real^m_{++}
\qquad \text{and} \qquad
W=\I_{n+m}. 
\]
These particular choices form a dual feasible solution that also satisfies \cref{eq:appstrictdualfeasibility}. Hence Slater condition holds for dual SDP as well. This implies that strong duality holds and the primal and dual SDP achieve the optimum value.
This completes the proof of \Cref{lem:slater}.
\end{proof}

\paragraph{Reformulating the dual.} Now we proceed to reformulate the dual to get to the form we use in our lower bounds. For any dual feasible solution $W$, we can write
\[
\Phi^*(W) = 
\begin{pmatrix}
0 & Y \\
Y^* & 0
\end{pmatrix}
=: \widehat Y.
\]
Given that $\widehat A$ has the same block diagonal structure as $\Phi^*(W)$ for any $W$, we have $\ip{\widehat A}{W} = \ip{\widehat A}{\Phi^*(W)}$. This leads to the following reformulation of the dual.
\begin{align}
    \label{eq:dualreformulation1}
    \begin{split}
    \gamma_{\op{F}}(A):= \max \quad & \ip{\widehat A}{\widehat Y} \\
        \text{s.t.} \quad 
        & \begin{pmatrix}
            \beta\I_n & 0 \\
            0 & \Delta_m(y)
          \end{pmatrix} \succeq \widehat Y = 
          \begin{pmatrix}
            0 & Y \\
            Y^* & 0
          \end{pmatrix} \\ 
        & \beta + \sum_{i\in [m]} y[i] = 1 \\
        & \widehat Y \in \herm{\complex^{n+m}}, \; \beta\in\real_+, \; \text{and} \; y \in \real^m_+.
        \end{split}
\end{align}
Finally, we reformulate the above dual into the form stated in Section~\ref{sec:sdpgammanorm}. To begin, we can safely assume that $\beta \in \real_{++}$ and $y \in \real^m_{++}$. Let $w,x\in\real^{n+m}_{++}$ be defined as
\begin{align*}
x[i] = \begin{cases}
\sqrt{\frac{n}{\beta}} & \; i \in [n] \\
\frac{1}{\sqrt{y[j]}} & \; j\in[m] \text{ and } i=n+j, 
\end{cases} \qquad \text{and} \qquad w[i] = \frac{1}{x[i]}.
\end{align*}
It is clear that 
\[
w = \begin{pmatrix} w_1 \\ w_2 \end{pmatrix} \in \real^{n+m}_{++}
\qquad \text{such that} \qquad 
\norm{w}_2 = 1 \quad \text{and } \quad 
w_1 = \alpha 1_n \; \text{for} \; \alpha = \sqrt{\beta/n}.
\]

\noindent The definition of $x$ and $w$ allows us to get an equivalent form of the first dual constraint in \cref{eq:dualreformulation1}, which is stated below.
\[
\begin{pmatrix}
\beta\I_n & 0 \\
0 & \Delta(y)
\end{pmatrix} \succeq \widehat Y
\qquad \text{if and only if} \qquad
\begin{pmatrix}
n\I_n & 0 \\
0 & \I_m
\end{pmatrix} \succeq 
\Delta_{n+m}(x) \widehat Y \Delta_{n+m}(x)
= 
\widehat Y \bullet xx^*
\]
where the final equality follows from Proposition~\ref{prop:schurdiagonalequiv}. Let us define
$\widehat Z := \widehat Y \bullet xx^*$. Using Proposition~\ref{prop:schurhermitianunitvector}, we have
\[
\ip{\widehat A}{\widehat Y} = \ip{\widehat A}{\widehat Y \bullet xx^* \bullet ww^*}
= \ip{\widehat A}{\widehat Z \bullet ww^*} = w^* (\widehat A\bullet \widehat Z)w.
\]

\noindent Hence the dual can be reformulated as the following optimization problem:
\begin{align}
\label{eq:appdualfinalform}
    \begin{split}
    \gamma_{\op{F}}(A)=    \max \quad & w^* (\widehat A \bullet \widehat Z) w \\
        \text{s.t.} \quad & \begin{pmatrix} n\I_n & 0 \\ 0 &  \I_m \end{pmatrix} \succeq \widehat Z \\
        & w = \begin{pmatrix}
        w_1 \\ w_2
        \end{pmatrix} \quad \text{such that} \qquad \norm{w}_2 = 1 
        \; \text{and} \; w_1 = \alpha1_n \\
        & \widehat Z \in \herm{\complex^{n+m}}, \; \alpha\in\real_{++},\; \text{and} \; w \in \real^{n+m}_{++}.
    \end{split}
\end{align}

\section{Useful Properties and Bounds on $\gamma_{\op{F}}(.)$}\label{app:gammanorm}
In this section, we establish few facts about $\gamma_{\op{F}}(.)$. 
Recall that, for any matrix $A\in\complex^{n\times m}$, $\gamma_{\op{F}}(A)$ can be written as an SDP where one minimizes a real parameter $\eta$ over $X\in\pos{\complex^{n+m}}$ such that
\begin{align}\label{eq:primalconstraints}
    \sum_{i=1}^n X[i,i] \leq \eta 
    \qquad \text{and} \qquad 
    \Phi(X) = X\bullet \widehat J_{n,m} = \widehat A
    \qquad \text{and} \qquad
    X[i,i] \le \eta \; \forall i\in[n+1, n+m].
\end{align}

\subsection{Useful Properties of $\gamma_{\op{F}}(.)$}
We first show that $\gamma_{\op{F}}(.)$ is indeed a norm. While this is known (Nikolov~\cite{nikolov2022} personally communicated a proof of this to us), we provide an (arguably) simpler proof for completeness. 

\begin{fact}
\label{fact:norm}
$\gamma_{\op{F}}(.)$ is a norm.
\end{fact}

\begin{proof}
Let $A\in\complex^{n\times m}$ be an arbitrary matrix. It is clear that $\gamma_{\op{F}}(A) = 0$ if and only if $A=0$ and $\gamma_{\op{F}}(\alpha A) = \abs{\alpha}\gamma_{\op{F}}(A)$ for any $\alpha\in \complex$. To see why the triangle inequality holds, let $A_1$ and $A_2$ be two matrices such that $A = A_1 + A_2$.
Let $(\eta_1,X_1)$ and $(\eta_2, X_2)$ be optimal solution for the SDPs corresponding to $A_1$ and $A_2$. For $A=A_1+A_2$, it is clear that $(\eta_1+\eta_2, X_1+X_2)$ is a feasible solution for the SDP. This implies that $\gamma_{\op{F}}(A_1+A_2) \leq \eta_1+\eta_2 = \gamma_{\op{F}}(A_1)+\gamma_{\op{F}}(A_2)$.
This completes the proof of Fact~\ref{fact:norm}.
\end{proof}

\begin{fact}\label{fact:gammarealmatrix}
For a matrix $A\in\real^{n\times m}$, $\gamma_{\op{F}}(A)$ is achieved by a real factorization of $A$.
\end{fact}

\begin{proof}
Let $B\in\complex^{n\times p}$ and $C\in\complex^{p\times m}$ be an optimal factorization of $A$. That is, $\gamma_{\op{F}}(A) =  \norm{B}_{\op{F}}\norm{C}_{1\rightarrow 2}$. Let 
\[
B = B_1 + \iota B_2 \qquad \text{and} \qquad C = C_1 + \iota C_2    
\]
for real matrices $B_1, B_2 \in \real^{n\times p}$ and $C_1, C_2 \in \real^{p\times m}$. Since $A$ is a real matrix, we have that 
\begin{align*}
    A = B_1 C_1 - B_2 C_2 = \begin{pmatrix} B_1 & -B_2 \end{pmatrix} 
    \begin{pmatrix} C_1 \\ C_2 \end{pmatrix}.
\end{align*}
Moreover, the above real factorization achieves $\gamma_{\op{F}}(A)$ completing the proof of Fact~\ref{fact:gammarealmatrix}.
\end{proof}

\noindent We now establish a bound on the dimension of matrix $B$ that gives an optimal factorization for $A = BC$ achieving $\gamma_{\op{F}}(A) = \norm{B}_{\op{F}}\norm{C}_{1\rightarrow 2}$.

\begin{lem}\label{lem:boundonoptfactorization}
Let $A\in\complex^{n\times m}$. Then one can construct $\widetilde B\in\complex^{n\times p}$ and $\widetilde C\in\complex^{p\times m}$ such that $p\le m$ and
\[
A = \widetilde B \widetilde C \qquad \text{and} \qquad
\gamma_{\op{F}}(A) = \norm{\widetilde B}_{\op{F}}\norm{\widetilde C}_{1\rightarrow 2}.
\]
\end{lem}

\begin{proof}
The quantity $\gamma_{\op{F}}(A)$ can be written as finding the optimal $\eta$ such that $X\in\pos{\complex^{n+m}}$ satisfying the constraints given in \cref{eq:primalconstraints}. Let $(\eta, X)$ be a feasible solution. We first show that one can recover an optimal factorization $A = \widetilde B \widetilde C$ with respect to $\gamma_{\op{F}}(A)$, where the number of columns of $\widetilde B$ is at most $m$. Given that $ X \succeq 0$, we can write
\begin{align}\label{eq:blockformprimal}
 X =
        \begin{pmatrix}  B \\ C^* \end{pmatrix} 
        \begin{pmatrix}  B \\ C^* \end{pmatrix}^*
        = \begin{pmatrix}  BB^* &  BC \\ (BC)^* & C^*C \end{pmatrix}.
\end{align}
Since $X$ is a feasible solution, we have that $\trace(BB^*) = \eta$ and $A=BC$. We will now construct a feasible solution of the form
\[
\widetilde X = \begin{pmatrix}  
\widetilde B \widetilde B^* & \widetilde B \widetilde C \\ 
(\widetilde B \widetilde C)^* & \widetilde C^* \widetilde C \end{pmatrix}
\]
such that $\trace(\widetilde B \widetilde B^*) \leq \trace(BB^*) = \eta$ and $\widetilde C^* \widetilde C = C^*C$,. Moreover, $A=\widetilde B \widetilde C$ where the number of columns of $\widetilde B$ is $m$.

Since $X\in\pos{\complex^{n+m}}$, there exist matrices $B\in\complex^{n \times r}$ and $C\in\complex^{r\times m}$ for $r \le n+m$ such that \cref{eq:blockformprimal} holds. While it is safe to assume that $r \le n+m$, any finite-dimensional choice of $r$ will also work for the argument presented next. If $r\le m$, then let $p=r$ and we are done with the proof of the lemma. 
Hence, for the remainder of the proof, assume that $r > m$. 
Let
\[
C = U\Sigma_C V^* \in \complex^{r \times m}
\]
be the singular value decomposition of $C$.  Since the number of singular values of $C$ is at most $\min \set{n,m}$,  we have $\Sigma_C \in \pd{\complex^p}$ for some $p  \le m < r$. It follows that $U \in \complex^{r \times p}$ and $V \in \complex^{m \times p}$.
By construction, we also have 
\[
U^*U = \I_p \qquad \text{and} \qquad UU^* \preceq \I_r.
\]
Let 
\[
\widetilde B = BU \qquad \text{and} \qquad \widetilde C = \Sigma_C V^* .
\]
It is clear that $\widetilde B \in\complex^{n\times p}$ and $\widetilde C \in  \complex^{p \times m}$ for $p \le m$, and $\widetilde B \widetilde C = BC = A$. Moreover, 
\[
C^*C = (U\Sigma_C V^*)^*(U\Sigma_C V^*) = V\Sigma^2_C V^* = 
(\Sigma_C V)^*\Sigma_C V = \widetilde C^* \widetilde C.
\]
Finally 
$\widetilde B \widetilde B^* = BUU^* B^* \preceq BB^\top$ since $UU^\top \preceq \I_r$. Therefore,
\begin{align*}
    \trace(\widetilde B \widetilde B^*)  \leq \trace(BB^*) \leq \eta.
\end{align*}
This completes the proof of \Cref{lem:boundonoptfactorization}.
\end{proof}

\noindent A similar result holds for any real matrix as stated below. 
\begin{cor}\label{cor:boundonrealoptfactorization}
Let $A\in\real^{n\times m}$. Then one can construct $B\in\real^{n\times p}$ and $C\in\real^{p\times m}$ for $p\le m$ such that $A=BC$ and $\gamma_{\op{F}}(A) = \norm{B}_{\op{F}}\norm{C}_{1\rightarrow 2}$.
\end{cor}

\noindent \Cref{cor:boundonrealoptfactorization} follows from Fact~\ref{fact:gammarealmatrix} and working out the proof of Lemma~\ref{lem:boundonoptfactorization} using matrix decompositions involving real matrices only. We first convert an optimal (possibly complex) factorization into a real factorization using Fact~\ref{fact:gammarealmatrix}. Let the optimal real factorization matrices be $B'\in\real^{n\times r}$ and $C'\in\real^{r \times m}$. Fact~\ref{fact:gammarealmatrix} guarantees that $r \le 2(n+m)$. As mentioned in the proof above, the proof works as long as $r$ is finite, which is the case here. 


\subsection{Useful Bounds on $\gamma_{\op F}(.)$}
\label{sec:boundsgammanorm}
\noindent A consequence of \Cref{lem:gammanormlowerboundmain} is the following result for special classes of square matrices.
\begin{cor}
\label{cor:gammanormsamesingularval}
Let $A\in\complex^{n\times n}$ be a non-singular matrix with one singular value of multiplicity $n$. Then $\gamma_{\op{F}}(A) = \norm{A}_{\op{F}}$.
\end{cor}
\begin{proof}
For a non-singular matrix $A$, if all its singular values are same, then $\norm{A}_1 = \sqrt{n}\norm{A}_{\op{F}}$. Hence $\gamma_{\op{F}}(A) = \norm{A}_{\op{F}}$ completing the proof of \Cref{cor:gammanormsamesingularval}.
\end{proof}

\noindent In particular, for any matrix $A\in\unitary{\complex^n}$, we have $\gamma_{\op{F}}(A) = \sqrt{n}$. 
A natural question to ask in which cases the lower and upper bound are tight for $\gamma_{\op{F}}(.)$ when they are not equal. Below, we give a partial answer to this question.

\begin{lem}
\label{lem:diagonalgamma}
Let $A\in\complex^{n\times n}$ be a diagonal matrix. Then $\gamma_{\op{F}}(A) = \norm{A}_{\op F}$.
\end{lem}

\begin{proof}
Our proof relies on constructing a dual feasible solution that achieves the objective value $\norm{A}_{\op F}$. By \cref{eq:gammaupperbound} and strong duality of the SDP (\Cref{lem:slater}), we will have $\gamma_{\op F}(A) = \norm{A}_{\op F}$. For this particular result, we will employ the dual formulation as described in \cref{eq:dualreformulation1} and construct a  feasible solution for it. In particular, we are looking for a matrix $Y\in\complex^{n\times n}$, $y\in\real^n_+$, and $\beta \ge 0$ such that
\[
\beta + \sum_{i=1}^n y[i] = 1
\qquad \text{and} \qquad 
\begin{pmatrix}
\beta\I_n & 0 \\
0 & \Delta_n(y)
\end{pmatrix}
\succeq \widehat Y = 
\begin{pmatrix}
0 & Y \\
Y^* & 0
\end{pmatrix}.
\]
Recall that $\Delta_n(y)$ is a linear map that maps an $n$-dimensional vector $y$ into a $n \times n$-dimensional diagonal matrix.
Since $A$ is a diagonal matrix, its singular values are the absolute values of its diagonal entries. Now we construct our dual feasible solution. Let $\beta = 1/2$, and 
\[
Y = \frac{A}{2\norm{A}_{\op F}}
\qquad \text{and} \qquad
\Delta_n(y) = \frac{A^*A}{2\norm{A}_{\op F}^2}.
\]
This implies that 
$y[i] = \frac{\abs{A[i,i]}^2}{2\norm{A}_{\op F}^2} = \frac{\sigma_i(A)^2}{2\norm{A}_{\op F}^2}$.
Now
\[
\beta + \sum_{i=1}^n y[i] = \frac{1}{2} + \frac{1}{2\norm{A}_{\op F}^2} \sum_{i=1}^n \sigma_i(A)^2 = 1.
\]
For our particular choice of $Y$, we have (using \Cref{lem:schurcomplement})
\begin{align*}
\label{eq:schurcompongeneraldualusage}
\begin{pmatrix}
\beta\I_n & -Y \\
-Y^* & \Delta_n(y)
\end{pmatrix} \succeq 0
\qquad \text{if and only if} \qquad
\Delta_n(y) - \beta^{-1}Y^*Y = \Delta_n(y) - 2Y^*Y\succeq 0.
\end{align*}
Since
\[
2Y^*Y = \frac{A^*A}{2\norm{A}_{\op F}^2} = \Delta_n(y),
\]
we have $\Delta_n(y) - 2Y^*Y = 0$, and therefore,
\[
\begin{pmatrix}
\beta\I_n & 0 \\
0 & \Delta_n(y)
\end{pmatrix} \succeq \widehat Y.
\]
This implies that all the dual constraints are satisfied. Next, we proceed to compute the objective value corresponding to this dual feasible solution which is at most $\gamma_{\op F}(A)$. We have
\[
\gamma_{\op F}(A) \ge \ip{\widehat A}{\widehat Y} = \trace(AY^*) + \trace(A^*Y) = \frac{1}{\norm{A}_{\op F}}\sum_{i\in[n]}\abs{A[i,i]}^2 = \norm{A}_{\op F}.
\]
This completes the proof of \Cref{lem:diagonalgamma}.
\end{proof}


\begin{thebibliography}{DMR{\etalchar{+}}22}

\bibitem[AFT22]{asi2022optimal}
Hilal Asi, Vitaly Feldman, and Kunal Talwar.
\newblock Optimal algorithms for mean estimation under local differential
  privacy.
\newblock {\em arXiv preprint arXiv:2205.02466}, 2022.

\bibitem[BDKT12]{bhaskara2012unconditional}
Aditya Bhaskara, Daniel Dadush, Ravishankar Krishnaswamy, and Kunal Talwar.
\newblock Unconditional differentially private mechanisms for linear queries.
\newblock In {\em Proceedings of the forty-fourth annual ACM Symposium on
  Theory of computing}, pages 1269--1284, 2012.

\bibitem[BEM{\etalchar{+}}17]{prochlo}
Andrea Bittau, {\'U}lfar Erlingsson, Petros Maniatis, Ilya Mironov, Ananth
  Raghunathan, David Lie, Mitch Rudominer, Ushasree Kode, Julien Tinnes, and
  Bernhard Seefeld.
\newblock Prochlo: Strong privacy for analytics in the crowd.
\newblock In {\em Proc. of the 26th {ACM} Symp. on Operating Systems Principles
  ({SOSP}'17)}, 2017.

\bibitem[Ben77]{bennett1977schur}
G~Bennett.
\newblock Schur multipliers.
\newblock 1977.

\bibitem[BST14]{bassily2014private}
Raef Bassily, Adam Smith, and Abhradeep Thakurta.
\newblock Private empirical risk minimization: Efficient algorithms and tight
  error bounds.
\newblock In {\em 2014 IEEE 55th Annual Symposium on Foundations of Computer
  Science}, pages 464--473. IEEE, 2014.

\bibitem[BV04]{boyd2004convexopt}
Stephen Boyd and Lieven Vandenberghe.
\newblock {\em Convex Optimization}.
\newblock Cambridge University Press, 2004.

\bibitem[Cen]{Census2021}
Differential privacy for census data explained.
\newblock
  \url{https://www.ncsl.org/research/redistricting/differential-privacy-for-census-data-explained.aspx}.
\newblock accessed: 2022-07-05.

\bibitem[CLSX12]{chan2012differentially}
T-H~Hubert Chan, Mingfei Li, Elaine Shi, and Wenchang Xu.
\newblock Differentially private continual monitoring of heavy hitters from
  distributed streams.
\newblock In {\em International Symposium on Privacy Enhancing Technologies
  Symposium}, pages 140--159. Springer, 2012.

\bibitem[CQ05]{chen2005best}
Chao-Ping Chen and Feng Qi.
\newblock The best bounds in wallis? inequality.
\newblock {\em Proceedings of the American Mathematical Society},
  133(2):397--401, 2005.

\bibitem[CR21]{cardoso2021differentially}
Adrian Cardoso and Ryan Rogers.
\newblock Differentially private histograms under continual observation:
  Streaming selection into the unknown.
\newblock {\em arXiv preprint arXiv:2103.16787}, 2021.

\bibitem[CSS11]{chan2011private}
T.{-}H.~Hubert Chan, Elaine Shi, and Dawn Song.
\newblock Private and continual release of statistics.
\newblock {\em {ACM} Trans. Inf. Syst. Secur.}, 14(3):26:1--26:24, 2011.

\bibitem[DMR{\etalchar{+}}22]{mcmahan2022private}
Sergey Denisov, Brendan McMahan, Keith Rush, Adam Smith, and Abhradeep
  Thakurta.
\newblock Improved differential privacy for sgd via optimal private linear
  operators on adaptive streams.
\newblock {\em arXiv preprint arXiv:2202.08312}, 2022.

\bibitem[DNPR10]{Dwork-continual}
Cynthia Dwork, Moni Naor, Toniann Pitassi, and Guy~N. Rothblum.
\newblock Differential privacy under continual observation.
\newblock In {\em Proc. of the Forty-Second {ACM} Symp. on Theory of Computing
  ({STOC}'10)}, pages 715--724, 2010.

\bibitem[DTTZ14]{dwork2014analyze}
Cynthia Dwork, Kunal Talwar, Abhradeep Thakurta, and Li~Zhang.
\newblock Analyze gauss: optimal bounds for privacy-preserving principal
  component analysis.
\newblock In {\em Proceedings of the forty-sixth annual ACM Symposium on Theory
  of computing}, pages 11--20. ACM, 2014.

\bibitem[ENU20]{edmonds2020power}
Alexander Edmonds, Aleksandar Nikolov, and Jonathan Ullman.
\newblock The power of factorization mechanisms in local and central
  differential privacy.
\newblock In {\em Proceedings of the 52nd Annual ACM SIGACT Symposium on Theory
  of Computing}, pages 425--438, 2020.

\bibitem[FHO21]{fichtenberger2021differentially}
Hendrik Fichtenberger, Monika Henzinger, and Wolfgang Ost.
\newblock Differentially private algorithms for graphs under continual
  observation.
\newblock In {\em 29th Annual European Symposium on Algorithms, {ESA} 2021,
  September 6-8, 2021, Lisbon, Portugal (Virtual Conference)}, 2021.

\bibitem[FHU22]{henzinger2022constant}
Hendrik Fichtenberger, Monika Henzinger, and Jalaj Upadhyay.
\newblock Constant matters: Fine-grained complexity of differentially private
  continual observation using completely bounded norms.
\newblock {\em arXiv preprint arXiv:2202.11205}, 2022.

\bibitem[Haa80]{haagerup1980decomposition}
Uffe Haagerup.
\newblock Decomposition of completely bounded maps on operator algebras, 1980.

\bibitem[Haz19]{hazan2019introduction}
Elad Hazan.
\newblock Introduction to online convex optimization.
\newblock {\em arXiv preprint arXiv:1909.05207}, 2019.

\bibitem[HLL{\etalchar{+}}22]{han2022private}
Yuxuan Han, Zhicong Liang, Zhipeng Liang, Yang Wang, Yuan Yao, and Jiheng
  Zhang.
\newblock Private streaming sco in $\ell_p$ geometry with applications in high
  dimensional online decision making.
\newblock In {\em International Conference on Machine Learning}, pages
  8249--8279. PMLR, 2022.

\bibitem[Hon15]{honaker2015efficient}
James Honaker.
\newblock Efficient use of differentially private binary trees.
\newblock {\em Theory and Practice of Differential Privacy (TPDP 2015), London,
  UK}, 2015.

\bibitem[HQYC21]{huang2021frequency}
Ziyue Huang, Yuan Qiu, Ke~Yi, and Graham Cormode.
\newblock Frequency estimation under multiparty differential privacy: One-shot
  and streaming.
\newblock {\em arXiv preprint arXiv:2104.01808}, 2021.

\bibitem[JRSS21]{jain2021price}
Palak Jain, Sofya Raskhodnikova, Satchit Sivakumar, and Adam Smith.
\newblock The price of differential privacy under continual observation.
\newblock {\em arXiv preprint arXiv:2112.00828}, 2021.

\bibitem[KMS{\etalchar{+}}21]{kairouz2021practical}
Peter Kairouz, Brendan McMahan, Shuang Song, Om~Thakkar, Abhradeep Thakurta,
  and Zheng Xu.
\newblock Practical and private (deep) learning without sampling or shuffling.
\newblock In {\em International Conference on Machine Learning}, pages
  5213--5225. PMLR, 2021.

\bibitem[KRSU10]{kasiviswanathan2010price}
Shiva~Prasad Kasiviswanathan, Mark Rudelson, Adam Smith, and Jonathan Ullman.
\newblock The price of privately releasing contingency tables and the spectra
  of random matrices with correlated rows.
\newblock In {\em Proceedings of the forty-second ACM Symposium on Theory of
  computing}, pages 775--784, 2010.

\bibitem[LM13]{li2013optimal}
Chao Li and Gerome Miklau.
\newblock Optimal error of query sets under the differentially-private matrix
  mechanism.
\newblock In {\em Proceedings of the 16th International Conference on Database
  Theory}, pages 272--283, 2013.

\bibitem[LMH{\etalchar{+}}15]{li2015matrix}
Chao Li, Gerome Miklau, Michael Hay, Andrew McGregor, and Vibhor Rastogi.
\newblock The matrix mechanism: optimizing linear counting queries under
  differential privacy.
\newblock {\em The VLDB journal}, 24(6):757--781, 2015.

\bibitem[Mat93]{mathias1993hadamard}
Roy Mathias.
\newblock The hadamard operator norm of a circulant and applications.
\newblock {\em SIAM journal on matrix analysis and applications},
  14(4):1152--1167, 1993.

\bibitem[McM17]{mcmahan2017survey}
H~Brendan McMahan.
\newblock A survey of algorithms and analysis for adaptive online learning.
\newblock {\em The Journal of Machine Learning Research}, 18(1):3117--3166,
  2017.

\bibitem[MK04]{merikoski2004inequalities}
Jorma~K Merikoski and Ravinder Kumar.
\newblock Inequalities for spreads of matrix sums and products.
\newblock {\em Applied Mathematics E-Notes}, 4:150--159, 2004.

\bibitem[MT22]{mcmahan2022federated}
Brendan McMahan and Abhradeep Thakurta.
\newblock Federated learning with formal differential privacy guarantees.
\newblock 2022.

\bibitem[Nik22]{nikolov2022}
Alexander Nikolov.
\newblock Personal communication.
\newblock 2022.

\bibitem[Pau82]{paulsen1982completely}
Vern~I Paulsen.
\newblock Completely bounded maps on $c^*$-algebras and invariant operator
  ranges.
\newblock {\em Proceedings of the American Mathematical Society}, 86(1):91--96,
  1982.

\bibitem[SM14]{strang2014functions}
Gilbert Strang and Shev MacNamara.
\newblock Functions of difference matrices are toeplitz plus hankel.
\newblock {\em siam REVIEW}, 56(3):525--546, 2014.

\bibitem[STU17]{smith2017interaction}
Adam Smith, Abhradeep Thakurta, and Jalaj Upadhyay.
\newblock Is interaction necessary for distributed private learning?
\newblock In {\em IEEE Symposium on Security and Privacy}, 2017.

\bibitem[TS13]{guha2013nearly}
Abhradeep Thakurta and Adam Smith.
\newblock (nearly) optimal algorithms for private online learning in
  full-information and bandit settings.
\newblock {\em Advances in Neural Information Processing Systems}, 26, 2013.

\bibitem[Upa19]{upadhyay2019sublinear}
Jalaj Upadhyay.
\newblock Sublinear space private algorithms under the sliding window model.
\newblock In {\em International Conference on Machine Learning}, pages
  6363--6372, 2019.

\bibitem[UU21]{upadhyay2021framework}
Jalaj Upadhyay and Sarvagya Upadhyay.
\newblock A framework for private matrix analysis in sliding window model.
\newblock In {\em International Conference on Machine Learning}, pages
  10465--10475. PMLR, 2021.

\bibitem[UUA21]{upadhyay2021differentially}
Jalaj Upadhyay, Sarvagya Upadhyay, and Raman Arora.
\newblock Differentially private analysis on graph streams.
\newblock In {\em International Conference on Artificial Intelligence and
  Statistics}, pages 1171--1179. PMLR, 2021.

\end{thebibliography}
\end{document}